%% file: main.tex
\documentclass[11pt]{article}

\usepackage{url}
\usepackage{amssymb,enumerate}
\usepackage{amsmath,amsfonts}
\usepackage{amsthm}
\usepackage{latexsym}
\usepackage{mathrsfs}
\usepackage{color}
\usepackage{microtype}
\usepackage{geometry}

\usepackage{algorithm}
\usepackage{algorithmicx}
\usepackage{algpseudocode}
\usepackage{enumitem}
\usepackage{empheq}
\usepackage{xcolor}
\usepackage{hyperref}
\usepackage{multirow}
\usepackage{appendix}
\usepackage[capitalise,nameinlink]{cleveref}

\input{math_commands.tex}

\hypersetup{backref=false,       
    pagebackref=true,               
    hyperindex=true,                
    colorlinks=true,                
    breaklinks=true,                
    urlcolor= black,                
    linkcolor= blue,                
    bookmarks=true,                 
    bookmarksopen=false,
    filecolor=black,
    citecolor=blue,
    linkbordercolor=blue
}

\theoremstyle{plain}
\newtheorem{theorem}{Theorem}

\newtheorem{lemma}[theorem]{Lemma}

\newtheorem{assumption}{Assumption}
\theoremstyle{definition}
\newtheorem{definition}[theorem]{Definition}

\newtheorem{remark}[theorem]{Remark}

\usepackage{graphicx, color}
\usepackage{algorithm, algpseudocode} 
\usepackage{mathrsfs} 

\usepackage{lipsum}

\title{Federated Learning with Convex Global and Local Constraints}
\author{Chuan He\thanks{Department of Computer Science and Engineering, University of Minnesota, USA (email: \url{he000233@umn.edu}; \url{peng0347@umn.edu}; \url{jusun@umn.edu}).}\and Le Peng$^*$ \and Ju Sun$^*$}

\date{
\today
}

\begin{document}
\maketitle
	
\begin{abstract}
In practice, many machine learning (ML) problems come with constraints, and their applied domains involve distributed sensitive data that cannot be shared with others, e.g., in healthcare. Collaborative learning in such practical scenarios entails federated learning (FL) for ML problems with constraints, or \emph{FL with constraints} for short. Despite the extensive developments of FL techniques in recent years, these techniques only deal with unconstrained FL problems or FL problems with simple constraints that are amenable to easy  projections. There is little work dealing with FL problems with general constraints. To fill this gap, we take the first step toward building an algorithmic framework for solving FL problems with general constraints. In particular, we propose a new FL algorithm for constrained ML problems based on the proximal augmented Lagrangian (AL) method. Assuming convex objective and convex constraints plus other mild conditions, we establish the worst-case complexity of the proposed algorithm. Our numerical experiments show the effectiveness of our algorithm in performing Neyman-Pearson classification and fairness-aware learning with nonconvex constraints, in an FL setting. 
\end{abstract}

\noindent\textbf{Keywords:} Federated learning, constrained machine learning, augmented Lagrangian, complexity analysis, imbalanced classification, fairness-aware learning

\medskip

\noindent\textbf{Mathematics Subject Classification:} 65Y20 68W15  90C60

\section{Introduction}
Federated learning (FL) has emerged as a prominent distributed machine learning (ML) paradigm that respects data privacy by design and has found extensive applications in diverse domains~\cite{kairouz2021advances}. In FL, ML models are trained without centralized training data: local clients hold their local data and never directly share them with other clients or the central server. Given a global ML model to train, typical FL strategies consist of repeated local computation and central aggregation: in each round, each local client performs local computation of quantities of interest (e.g., local model parameters or derivatives) based on the local data, and then the central server collects and aggregates the local results and updates the parameters of the global ML model. Since the shared local results are usually highly nonlinear functions of local data, making reverse engineering of local data unlikely, data privacy is naturally protected. 

\subsection{Federated learning for constrained machine learning problems}
However, existing FL techniques are developed almost exclusively for unconstrained ML problems or, at best, for ML problems with simple constraints that are amenable to easy projections, despite the growing list of ML problems with general constraints---where constraints typically encode prior knowledge and desired properties, e.g., robustness evaluation~\cite{goodfellow2014explaining}, fairness-aware learning~\cite{ABDL18}, learning with imbalanced data~\cite{saito2015precision}, neural architecture search~\cite{zoph2018learning}, topology optimization~\cite{christensen2008introduction}, physics-informed machine learning~\cite{mcclenny2020self}. Here, we sketch two quick examples.

\paragraph{Neyman-Pearson classification, or optimizing the false-positive rate with a controlled false-negative rate} 
Conventional binary classification assumes equal importance in both classes, so predictive errors in both classes are counted equally. In numerous applications, such as medical diagnosis, misclassifying one class (i.e., the priority class) is much more costly than misclassifying the other. The Neyman-Pearson classification framework addresses this asymmetry in misclassification cost by explicitly controlling the error rate in the priority class while optimizing that in the other~\cite{TFZ16,scott2007performance,rigollet2011neyman,tong2018neyman}:
\begin{equation} \label{eq:np-clf}
\min_{\theta} \frac{1}{n_0} \sum_{i=1}^{n_0}\varphi(f_\theta,z_{i,0}) \quad \st\quad \frac{1}{n_1}\sum_{i=1}^{n_1}\varphi(f_\theta,z_{i,1})\le r,
\end{equation}
where $f_\theta$ is the trainable binary classifier parameterized by $\theta$, $\varphi$ is the loss function serving as a proxy to classification error, and $\{z_{i,0}\}_{i=1}^{n_0}$ and $\{z_{i,1}\}_{i=1}^{n_1}$ are the training data from class $0$ and $1$, respectively. The constraint imposes an upper bound on the error rate for class $1$. 

\paragraph{Fairness-aware learning} 
Typical ML models are known to have biases toward the majority subgroups of the input space~\cite{ABDL18,CHKV19,MMSLG21fair}. For example, a disease diagnostic model that is trained on a male-dominant dataset tends to predict much more accurately on the male subgroup than on the female subgroup. To counteract such potential model biases, a natural way is to enforce fairness constraints to ensure that the performance of the model on different subgroups is comparable~\cite{ABDL18,CHKV19,MMSLG21fair}. A possible formulation for two-subgroup problems is
\begin{equation}\label{eq:fair}
    \min_{\theta}\frac{1}{n^\prime}\sum_{i=1}^{n^\prime}\varphi(f_\theta,z_i)\quad \st\quad -\delta\le \frac{1}{|\mathcal{S}_0|}\sum_{i\in\mathcal{S}_0}\varphi(f_\theta,z_i)-\frac{1}{|\mathcal{S}_1|}\sum_{i\in\mathcal{S}_1}\varphi(f_\theta,z_i)\le \delta, 
\end{equation}
where $f_\theta$ is the ML model parameterized by $\theta$, $\{z_i\}$ is the training set, and $\varphi$ is the proxy loss function, similar to the setup in \cref{eq:np-clf}. With $\mathcal S_0$ and $\mathcal S_1$ denoting the two subgroups of interest, the constraint imposes that the performance disparity of $f_\theta$ on $\mathcal S_0$ and $\mathcal S_1$ should not be larger than $\delta>0$, which is usually set close to $0$.

Both examples are particularly relevant to biomedical problems, where class imbalance and subgroup imbalance are prevalent. Moreover, there are strict regulations on the distribution and centralization of biomedical data for research, e.g., the famous Health Insurance Portability and Accountability Act (HIPAA) protection of patient privacy~\cite{OCR}. Together, these underscore the importance of developing FL techniques for constrained ML problems, which is largely lacking: FL problems with only simple constraints that are amenable to easy projections have been considered in \cite{yuan2021federated,tran2021feddr}, and a small number of papers have tried to mitigate class imbalance~\cite{shen2021agnostic} and improve model fairness \cite{du2021fairness,chu2021fedfair,galvez2021enforcing} through constrained optimization in FL settings. However, these developments are specialized to their particular use cases and lack computational guarantees for the feasibility and optimality of their solutions. 

\emph{In this paper, we take the first step toward a general and rigorous FL framework for general constrained ML problems}. Consider the constrained ML problems in \cref{eq:np-clf,eq:fair} in an FL setting with $n$ local clients, where the $i^{th}$ client holds local data $Z_i$ and so the whole training set is the union $Z_1 \cup \dots \cup Z_n$. Since the objectives and constraints in \cref{eq:np-clf,eq:fair} are in the finite-sum form, both examples are special cases of the following finite-sum constrained optimization problem: 
\begin{equation}\label{finite-sum-cnstr}
    \min_{\theta} \left\{\sum_{i=1}^n f_i(\theta;Z_i) + h(\theta)\right\}\quad\st\quad  \underbrace{\sum_{i=1}^n\tilde{c}_i(\theta;Z_i)\le 0.}_{\text{local data coupled}}
\end{equation}
Note that inside the constraint, the local data are coupled which necessarily lead to communication between the local clients and the central server to allow the evaluation of the constraint function. To try to reduce such communication so that we can have more flexibility in algorithm design, we introduce decoupling variables $\{s_i\}$, leading to an equivalent formulation: 
\begin{equation}\label{simple-local}
    \min_{\theta,s_i} \left\{\sum_{i=1}^n f_i(\theta;Z_i) + h(\theta)\right\}\quad\st\quad \underbrace{\sum_{i=1}^ns_i \le 0,}_{\text{no local data}}\quad \underbrace{\tilde{c}_i(\theta;Z_i)\le s_i, \ \ 1\le i\le n}_{\text{local data decoupled}}, 
\end{equation}
which decouples the local data into $n$ local constraints. 

In this paper, we consider the following setup for FL with global and local constraints, a strict generalization of \cref{simple-local}: 
\begin{empheq}[box=\fbox]{align} \label{FL-gl-cone}
\min_{w} \left\{\sum_{i=1}^n f_i(w; Z_i) + h(w)\right\} \quad \st\quad \underbrace{c_0(w; Z_0)\le0}_{\text{global constraint}},\quad \underbrace{c_i(w; Z_i)\le 0,\ \ 1\le i\le n}_{\text{local constraints}}. 
\end{empheq}
In contrast to unconstrained FL or FL with simple constraints amenable to easy projections studied in
existing literature, our focus lies on general convex constraints, where projections may or may not be easy to
compute. Here, we assume $n$ local clients, each with a local objective $f_i(w; Z_i)$ and a set of local constraints $c_i(w; Z_i)\le 0$ (i.e., scalar constraints are vectorized) for $i = 1, \dots, n$. To stress the FL setting, we spell out the dependency of these local objectives and local constraints on the local data $Z_i$'s: \emph{these $Z_i$'s are only accessible to their respective local clients and should never be shared with other local clients or the central server}. Henceforth, we omit $Z_i$'s in local objectives and constraints when no confusion arises. To allow flexibility in modeling, we also include the global constraint $c_0(w; Z_0)\le 0$, with central data $Z_0$ that is only accessible by the central server. To facilitate the theoretical study of the algorithm that we develop, we further assume that all objective and scalar constraint functions are convex, but we also verify the applicability of our FL algorithm to classification problems with nonconvex fairness constraints in \cref{sec:fair-class}. To summarize, our standing assumptions on top of \cref{FL-gl-cone} include 
\begin{assumption}\label{asp:cvx-lclc}
We make the following assumption throughout this paper:
\vspace{-0.5em}
\begin{enumerate}[label=(\alph*),leftmargin=1.5em]
\item The objective functions $f_i:\bR^d\to\bR$, $1\le i\le n$, and the $m_i$ scalar constraint functions inside $c_i:\bR^d\to\bR^{m_i}$, $0\le i\le n$, are convex and continuously differentiable, and $h:\bR^d\to(-\infty,\infty]$ is a simple closed convex function.    
\item For each $1\le i\le n$, only the local objective $f_i$ and local constraint $c_i$ have access to the local data $Z_i$, which are never shared with other local clients and the central server. Only the central server has access to the global data $Z_0$,  
\end{enumerate}
\end{assumption}

\subsection{Our contributions}
This paper tackles \cref{FL-gl-cone} by adopting the sequential penalization approach, which involves solving a sequence of unconstrained subproblems that combine the objective function with penalization of constraint violations. In particular, we propose an FL algorithm based on the proximal augmented Lagrangian (AL) method developed in \cite{LZ18AL}. In each iteration, the unconstrained subproblem is solved by an inexact solver based on the alternating direction method of multipliers (ADMM) in a federated manner. We study the worst-case complexity of the proposed algorithm assuming {\it locally} Lipschitz continuous gradients. Our main contributions are highlighted below.

\begin{itemize}[leftmargin=1em]
\item We propose an FL algorithm (\cref{alg:g-admm-cvx}) for solving \cref{FL-gl-cone} based on the proximal AL method. To the best of our knowledge, the proposed algorithm is the first in solving general constrained ML problems in an FL setting. Assuming {\it locally} Lipschitz continuous gradients and other mild conditions, we establish its worst-case complexity to find an approximate optimal solution of \cref{FL-gl-cone}. The complexity results are entirely new in the literature. 

\item We propose an ADMM-based inexact solver in \cref{alg:admm-cvx-1} to solve the unconstrained subproblems arising in \cref{alg:g-admm-cvx}. We equip this inexact solver with a newly introduced verifiable termination criterion and establish its global linear convergence for solving the subproblems of \cref{alg:g-admm-cvx}; these subproblems are strongly convex and have locally Lipschitz continuous gradients.

\item We perform numerical experiments to compare our proposed FL algorithm (\cref{alg:g-admm-cvx}) with the centralized proximal AL method (\cref{alg:c-prox-AL}) on binary Neyman-Pearson classification and classification with nonconvex fairness constraints using real-world datasets (\cref{sec:exp}). Our numerical results demonstrate that our FL algorithm can achieve solution quality comparable to that of the centralized proximal AL method.
\end{itemize}

\subsection{Related work}
\paragraph{FL algorithms for unconstrained optimization} 
FL has emerged as a cornerstone for privacy-preserved learning since Google's seminal work~\cite{mcmahan2017communication}, and has found applications in numerous domains where the protection of data privacy precludes centralized learning, including healthcare~\cite{rieke2020future, peng2023evaluation, peng2023a}, finance~\cite{long2020federated}, Internet of things~\cite{mills2019communication}, and transportation~\cite{liu2020privacy}. FedAvg~\cite{mcmahan2017communication} is the first and also the most popular FL algorithm to date. After FedAvg, numerous FL algorithms have been proposed to improve performance and address practical issues, such as data heterogeneity~\cite{karimireddy2020scaffold,li2021fedbn,ZHDYL21fedpd}, system heterogeneity~\cite{li2020federated,wang2020tackling,GLF22fedadmm}, fairness~\cite{li2021ditto}, communication efficiency~\cite{sattler2019robust, konevcny2016federated,mishchenko2022proxskip}, convergence \cite{pathak2020fedsplit}, handling simple constraints \cite{yuan2021federated,tran2021feddr}, incentives~\cite{travadi2023welfare}, and hyperparameter tuning \cite{yao2024constrained}. Since our FL algorithm relies on applying an inexact ADMM (\cref{alg:admm-cvx-1}) to solve subproblems, it is also worth mentioning that ADMM-based algorithms have been proposed to handle FL problems \cite{ZL23ADMM,GLF22fedadmm,ZHDYL21fedpd} and optimization problems with many
constraints in a distributed manner \cite{giesen2019combining}. More FL algorithms and their applications can be found in the survey~\cite{li2021survey}. Despite the intensive research on FL, existing algorithms focus primarily on unconstrained ML problems, versus constrained ML problems considered in this paper.

\paragraph{Centralized algorithms for constrained optimization} 
Recent decades have seen fruitful algorithm developments for centralized constrained optimization in numerical optimization. In particular, there has been a rich literature on AL methods for solving convex constrained optimization problems \cite{aybat2013augmented,necoara2019complexity,patrascu2017adaptive,xu2021iteration,lan2016iteration,LZ18AL,LM22}. In addition, variants of AL methods have been developed to solve nonconvex constrained optimization problems \cite{hong2017prox,grapiglia2021complexity,birgin2020complexity,kong2023iteration,li2021rate,he2023BAL,he2023newton,lu2022single}. Besides AL methods and their variants, sequential quadratic programming methods \cite{boggs1995sequential,curtis2012sequential}, trust-region methods \cite{byrd1987trust,powell1991trust}, interior point methods \cite{wachter2006implementation}, and extra-point methods \cite{huang2022accelerated} have been proposed to solve centralized constrained optimization problems. 

\paragraph{Distributed algorithms for constrained optimization} 
Developing distributed algorithms for constrained optimization has started relatively recently. To handle simple local constraints in distributed optimization, \cite{nedic2010constrained,lin2016distributed,wang2017distributed} study distributed projected subgradient methods. For complicated conic local constraints, \cite{aybat2016primal,aybat2019distributed} develop distributed primal-dual algorithms. For distributed optimization with global and local constraints, \cite{zhu2011distributed,yuan2011distributed} develop primal-dual projected subgradient algorithms. For an overview of distributed constrained optimization, see \cite{yang2019survey}. Notice that FL is a special distributed optimization/learning framework that protects data privacy by prohibiting the transfer of raw data from one client to another or to a central server. These distributed algorithms for constrained optimization do not violate the FL restriction and hence can be considered as FL algorithms, but they can only handle problems with simple global or local constraints that are amenable to easy projection. Therefore, they cannot be applied directly to our setup~\cref{FL-gl-cone} with general global and local constraints.

\paragraph{FL algorithms for constrained ML applications} 
A small number of papers have developed FL algorithms for particular constrained ML applications, such as learning with class imbalance and fairness-aware ML. For example, \cite{shen2021agnostic,chu2021fedfair} propose FL algorithms to address class imbalance and subgroup imbalance, respectively, by optimizing the Lagrangian function. \cite{du2021fairness} applies quadratic penalty method to deal with the constraint in fairness-aware ML. In addition, \cite{galvez2021enforcing} proposes an FL algorithm to tackle fairness-aware ML based on optimizing the AL function. However, these developments are tailored to specific applications and lack rigorous computational guarantees regarding the feasibility and optimality of their solutions. In contrast, this paper focuses on developing algorithms with theoretical guarantees for FL with convex global and local constraints. To the best of our knowledge, this work provides the first general FL framework for constrained ML problems.

\section{Notation and preliminaries}
Throughout this paper, we let $\bR^d$ and $\bR^d_+$ denote the $d$-dimensional Euclidean space and its nonnegative orthant, respectively. We use $\langle\cdot,\cdot\rangle$ to denote the standard inner product, $\|\cdot\|$ to denote the Euclidean norm of a vector or the spectral norm of a matrix, and $\|\cdot\|_{\infty}$ to denote the $\ell_{\infty}$-norm of a vector. For any vector $v \in \bR^d$, $[v]_+ \in \bR^d$ is its nonnegative part (i.e., with all negative values set to zero). We adopt the standard big-O notation $\cO(\cdot)$ to present complexity results; $\widetilde{\cO}(\cdot)$ represents $\cO(\cdot)$ with logarithmic terms omitted. 

Given a closed convex function $h:\bR^d\to(-\infty,\infty]$, $\partial h$ and $\dom(h)$ denote the subdifferential and domain of $h$, respectively. The proximal operator associated with $h$ is denoted by $\mathrm{prox}_h$, that is, $\mathrm{prox}_h(u) = \argmin_{w}\{\|w-u\|^2/2+h(w)\}$ for all $u\in\bR^d$. Given a continuously differentiable mapping $\phi:\bR^d\to\bR^p$, we write the transpose of its Jacobian as $\nabla \phi(w)=\left[\nabla \phi_1(w)\ \cdots\ \nabla\phi_p(w)\right]\in\bR^{d\times p}$. We say that $\nabla \phi$ is $L$-Lipschitz continuous on a set $\Omega$ for some $L>0$ if $\|\nabla\phi(u)-\nabla\phi(v)\|\le L\|u-v\|$ for all $u,v\in\Omega$. In addition, we say that $\nabla\phi$ is locally Lipschitz continuous on $\Omega$ if for any $w\in\Omega$, there exist $L_{w}>0$ and an open set $\cU_w$ containing $w$ such that $\nabla \phi$ is $L_w$-Lipschitz continuous on $\cU_w$. 

Given a nonempty closed convex set $\cC\subseteq\bR^d$ and any point $u \in \bR^d$, $\rmdist(u,\cC)$ and $\rmdist_{\infty}(u,\cC)$ stand for the Euclidean distance and the Chebyshev distance from $u$ to $\cC$, respectively. That is, $\rmdist(u,\cC) = \min_{v\in\cC} \|u-v\|$ and $\rmdist_{\infty}(u,\cC)= \min_{v\in\cC} \|u-v\|_{\infty}$. The normal cone of $\cC$ at $u\in\cC$ is denoted by $\cN_{\cC}(u)$. The Minkowski sum of two sets $\cB$ and $\cC$ is defined as $\cB+\cC := \{b+c:b\in\cB,c\in\cC\}$.

For ease of presentation, we let $m := \sum_{i=0}^n m_i$ and adopt the following notations throughout this paper:
\begin{equation}\label{abbre}
f(w)=\sum_{i=1}^n f_i(w),\qquad 
c(w)=\left[\begin{smallmatrix} c_0(w) \\ \vdots \\ c_n(w)\end{smallmatrix}\right] \in \R^{m},
\qquad \mu=\left[\begin{smallmatrix}\mu_0 \\ \vdots \\ \mu_n \end{smallmatrix}\right] \in \R^{m}.
\end{equation}
\begin{assumption}\label{asp:basic}
Throughout this paper, we assume that the strong duality holds for \cref{FL-gl-cone} and its dual problem
\begin{equation}\label{FL-cons-dual}
\sup_{\mu\ge0}\inf_w\left\{f(w) + h(w) + \langle\mu,c(w)\rangle\right\}.
\end{equation}
That is, both problems have optimal solutions and, moreover, their optimal values coincide.    
\end{assumption}
Under \cref{asp:basic}, it is known that $(w,\mu)\in\dom(h)\times\bR^m_+$ is a pair of optimal solutions of \cref{FL-gl-cone} and \cref{FL-cons-dual} if and only if it satisfies (see, e.g., \cite{LZ18AL})
\begin{equation}\label{def:kkt}
0\in\begin{pmatrix}
\nabla f(w) +\partial h(w) + \nabla c(w)\mu\\
c(w) - \cN_{\bR^m_+}(\mu)
\end{pmatrix}.
\end{equation}
In general, it is hard to find an exact optimal solution of \cref{FL-gl-cone} and \cref{FL-cons-dual}. Thus, we are instead interested in seeking an approximate optimal solution of \cref{FL-gl-cone} and \cref{FL-cons-dual} defined as follows. 
\begin{definition}\label{def:eps-KKT}
Given any $\epsilon_1,\epsilon_2>0$, we say $(w,\mu)\in\dom(h)\times\bR^m_+$ is an $(\epsilon_1,\epsilon_2)$-optimal solution of \cref{FL-gl-cone} and \cref{FL-cons-dual} if $\rmdist_{\infty}\left(0, \nabla f(w)+\partial h(w)+\nabla c(w)\mu\right)\le\epsilon_1$ and $\rmdist_{\infty}(c(w),\cN_{\bR^m_+}(\mu))\le\epsilon_2$.\footnote{For unconstrained convex problems with differentiable objective $\min_{w} f(w)$, a natural measure of convergence is $\|\nabla f(w)\|$, i.e., the distance between $0$ and $\nabla f(w)$, as the optimality condition is $\nabla f(w) = 0$. If the objective is nondifferentiable, we need to use the notation of subdifferential, $\partial f(w)$, which is a set for each $w$ in general. In this case, the optimality condition reads $0 \in \partial f(w)$, and the measure of convergence is the distance between $0$ and the subdifferent set $\mathrm{dist}(0, \partial f(w)) := \min_{u\in\partial f(w)}\|u\|$.}
\end{definition}
Here, the two different tolerances $\epsilon_1,\epsilon_2$ are used for measuring stationarity and feasibility violation, respectively. This definition is consistent with the $\epsilon$-KKT solution considered in \cite{LZ18AL} except that \cref{def:eps-KKT} uses the Chebyshev distance rather than the Euclidean distance.

\section{A proximal AL based FL algorithm for solving \cref{FL-gl-cone}}\label{sec:d-pal}
In this section, we propose an FL algorithm for solving \cref{FL-gl-cone} based on the proximal AL method. Specifically, we describe this algorithm in \cref{sec:alg-cvx-obj}, and then analyze its complexity results in \cref{sec:complex-cvx-obj}. 
\begin{assumption}\label{asp:out-analysis}
Throughout this section, we assume that 
\vspace{-0.5em}
\begin{enumerate}[label=(\alph*),leftmargin=1.5em]
\item The proximal operator for $h$ can be exactly evaluated.
\item 
The gradients $\nabla f_i$, $1\le i\le n$, and the transposed Jocobians $\nabla c_i$, $0\le i\le n$, are locally Lipschitz continuous on $\bR^d$.
\end{enumerate}
\end{assumption}
\cref{asp:out-analysis}(b) clearly holds if all $\nabla f_i$'s and $\nabla c_i$'s are globally Lipschitz continuous on $\bR^d$, but this assumption holds for a broad class of problems without global Lipschitz continuity on $\nabla f_i$'s and $\nabla c_i$'s. For example, the quadratic penalty function of $c(w)\le 0$, namely $\|[c(w)]_+\|^2$, only has a locally Lipschitz continuous gradient even if $\nabla c$ is globally Lipschitz continuous on $\bR^d$ (see \cref{rmk:not-loc-lip}). In addition, the gradient of a convex high-degree polynomial, such as $\|w\|^4$ with $w\in\mathbb{R}^d$, is locally Lipschitz continuous but not globally Lipschitz continuous on $\mathbb{R}^d$.

\subsection{Algorithm description}\label{sec:alg-cvx-obj}
In this subsection, we describe a proximal AL-based FL algorithm (\cref{alg:g-admm-cvx}) for finding an $(\epsilon_1,\epsilon_2)$-optimal solution of \cref{FL-gl-cone} for prescribed $\epsilon_1,\epsilon_2\in(0,1)$. This algorithm follows a framework similar to a centralized proximal AL method described in \cref{apx:cpal}; see Section 11.K in \cite{rockafellar2009variational} or \cite{LZ18AL} for more details of proximal AL. At each iteration, it applies an inexact ADMM solver (\cref{alg:admm-cvx-1}) to find an approximate solution $w^{k+1}$ to the proximal AL subproblem associated with \cref{FL-gl-cone}:
\begin{equation}
\min_{w}\Bigg\{\ell_k(w):= \underbrace{\sum_{i=1}^n f_i(w) + h(w) + \frac{1}{2\beta}\sum_{i=0}^n\left(\|[\mu_i^k+\beta c_i(w)]_+\|^2-\|\mu_i^k\|^2\right)}_{\text{augmented Lagrangian function}} + \underbrace{\frac{1}{2\beta}\|w-w^k\|^2}_{\text{proximal term}}\Bigg\}.\label{prox-AL-pb}  
\end{equation}
Then, the multiplier estimates are updated according to the classical scheme:
\[
\mu_i^{k+1} = [\mu^k_i+\beta c_i(w^{k+1})]_+,\quad 0\le i\le n.
\]

\begin{algorithm}[h]
\caption{A proximal AL based FL algorithm for solving \cref{FL-gl-cone}}
\label{alg:g-admm-cvx}
\noindent\textbf{Input}: tolerances $\epsilon_1,\epsilon_2\in(0,1)$, $w^0\in\dom(h)$, $\mu_i^0\ge0$ for $0\le i\le n$, $\Bar{s}>0$, and $\beta>0$.
\begin{algorithmic}[1]
\For{$k=0,1,2,\ldots$}
\State Set $\tau_k=\Bar{s}/(k+1)^2$.\label{def:beg-outer-loop}
\State Call \cref{alg:admm-cvx-1} (see \cref{sec:admm} below) with $(\tau,\tilde{w}^0)=(\tau_k,w^k)$ to find an approximate solution $w^{k+1}$ 
\State to \cref{prox-AL-pb-compact} in a federated manner such that
\begin{equation}\label{cond:apx-stat}
\mathrm{dist}_{\infty}(0,\partial\ell_k(w^{k+1}))\le\tau_k.
\end{equation}
\State {\bf Server update:} The central server updates $\mu_0^{k+1}=[\mu_0^k + \beta c_0(w^{k+1})]_+$.
\State {\bf Communication (broadcast):} Each local client $i$, $1\le i\le n$, receives $w^{k+1}$ from the central server.
\State {\bf Client update (local):} Each local client $i$, $1\le i\le n$, updates $\mu_i^{k+1}=[\mu_i^k + \beta c_i(w^{k+1})]_+$.
\State {\bf Communication:} Each local client $i$, $1\le i\le n$, sends $\|\mu_i^{k+1}-\mu_i^k\|_\infty$ to the central server.
\State {\bf Termination (server side):} Output $(w^{k+1},\mu^{k+1})$ and terminate the algorithm if
\begin{equation}\label{stop-alg}
\|w^{k+1}-w^k\|_{\infty}+\beta\tau_k\le\beta\epsilon_1,\qquad \max_{0\le i\le n}\{\|\mu_i^{k+1}-\mu_i^k\|_{\infty}\}\le\beta\epsilon_2.
\end{equation} 
\label{def:end-outer-loop}
\EndFor
\end{algorithmic}
\end{algorithm}

Notice that the subproblem in \cref{prox-AL-pb} can be rewritten as 
\begin{equation}\label{prox-AL-pb-compact}
\min_w\left\{\ell_k(w) := \sum_{i=0}^n P_{i,k}(w) + h (w)\right\}, 
\end{equation}
where $P_{i,k}$, $0\le i\le n$, are defined as
\begin{align}
P_{0,k}(w):= &\ \frac{1}{2\beta}\left(\|[\mu_0^k+\beta c_0(w)]_+\|^2-\|\mu_0^k\|^2\right) + \frac{1}{2(n+1)\beta}\|w-w^k\|^2,\label{def:P0k}\\
P_{i,k}(w):= &\ f_i(w) + \frac{1}{2\beta}\left(\|[\mu_i^k+\beta c_i(w)]_+\|^2-\|\mu_i^k\|^2\right) + \frac{1}{2(n+1)\beta}\|w-w^k\|^2,\quad \forall 1\le i\le n.\label{def:Pik}
\end{align}
When \cref{alg:admm-cvx-1} (see \cref{sec:admm}) is applied to solve \cref{prox-AL-pb-compact}, the local merit function $P_{i,k}$, constructed from the local objective $f_i$ and local constraint $c_i$, is handled by the respective local client $i$, while the merit function $P_{0,k}$ is handled by the central server. We observe that \cref{alg:g-admm-cvx} with the subproblem in \cref{prox-AL-pb-compact} solved by \cref{alg:admm-cvx-1} meets the basic FL requirement: since local objective $f_i$'s and local constraint $c_i$'s are handled by their respective local clients and the central server only performs aggregation and handles the global constraint $c_0$, no raw data are shared between the local clients and the central server, i.e., \cref{asp:cvx-lclc}(b) is obeyed.

\begin{remark}
    We now make the following remarks on \cref{alg:g-admm-cvx}.
    \begin{enumerate}[label=(\alph*),leftmargin=*]
    \vspace{-0.5em}
    \item For hyperparameters of \cref{alg:g-admm-cvx}, 
    \begin{itemize}[leftmargin=*]
    \item $\epsilon_1,\epsilon_2\in(0,1)$ only depend on the numerical accuracy that the user aims to achieve;
    \item the initial iterates $w^0$ and $\mu^0_i$, $1\leq i \leq n$, are usually randomly generated or set as a constant vector;
    \item $\Bar{s}>0$ controls the tolerance sequence $\{\tau_k\}_{k\ge0}$ for the subproblems in \cref{alg:g-admm-cvx}. These finite, non-zero tolerances allow us to solve the subproblems \textit{inexactly} but can still guarantee convergence, hence saving computational costs. In particular, setting $\{\tau_k\}_{k\ge 0}$ to diminish rapidly towards zero on the order of $\mathcal{O}(1/k^2)$ can guarantee convergence of \cref{alg:g-admm-cvx}. In practice, $\Bar{s}$ only needs to be set as $\mathcal{O}(1)$. 
\end{itemize}
    \item Compared to the centralized proximal AL developed in \cite{LZ18AL}, we have made the following major changes to arrive at \cref{alg:g-admm-cvx}.
    \begin{itemize}[leftmargin=*] 
       \item add communication steps to allow dual updates in an FL manner; 
       \item to solve the subproblem, we cannot directly apply the accelerated gradient method (AGM) as in \cite{LZ18AL}. It is possible to develop an FL version of AGM by {\it eagerly} aggregating gradients from local clients, but that induces heavy communication between clients and the central server. To address this, we first reformulate the subproblem as a finite-sum problem and then propose an inexact ADMM solver to solve it. The inexact ADMM solver allows multiple steps of local updates before aggregation of model weights at the central server, hence it is communication friendly. We also propose a new stopping criterion for the inexact ADMM (\cref{alg:admm-cvx-1}). Detailed explanations can be found in \cref{rmk:alg2}.
    \end{itemize}
\end{enumerate}
\end{remark}

For ease of later reference, we refer to the update from $w^k$ to $w^{k+1}$ as one outer iteration of \cref{alg:g-admm-cvx}, and call one iteration of \cref{alg:admm-cvx-1} for solving \cref{prox-AL-pb} one inner iteration of \cref{alg:g-admm-cvx}. In the rest of this section, we study the following measures of complexity for \cref{alg:g-admm-cvx}. 
\begin{itemize}[leftmargin=*]
\item {\it Outer iteration complexity}, which measures the number of outer iterations of \cref{alg:g-admm-cvx} ({\it one outer iteration} refers to one execution from \cref{def:beg-outer-loop} to \cref{def:end-outer-loop} in \cref{alg:g-admm-cvx});
\item {\it Total inner iteration complexity}, which measures the total number of iterations of \cref{alg:admm-cvx-1} that are performed in \cref{alg:g-admm-cvx} ({\it one inner iteration} refers to one execution from \cref{def:for-loop-beg} to \cref{def:for-loop-end} in \cref{alg:admm-cvx-1}).
\end{itemize}
The following theorem concerns the output of \cref{alg:g-admm-cvx}, whose proof is deferred to \cref{apx:out-pal}.

\begin{theorem}[{{\bf output of \cref{alg:g-admm-cvx}}}]\label{thm:output-pal}
If \cref{alg:g-admm-cvx} successfully terminates, its output $(w^{k+1},\mu^{k+1})$ is an $(\epsilon_1,\epsilon_2)$-optimal solution of \cref{FL-gl-cone}.
\end{theorem}

\subsection{Complexity analysis}\label{sec:complex-cvx-obj}
In this subsection, we establish the outer and total inner iteration complexity for \cref{alg:g-admm-cvx}. To proceed, we let $(w^*,\mu^*)$ be any pair of optimal solutions of \cref{FL-gl-cone} and \cref{FL-cons-dual}. First, we establish a lemma to show that all iterates generated by \cref{alg:g-admm-cvx} are bounded. Its proof can be found in \cref{apx:bd-pal}.
\begin{lemma}[{{\bf bounded iterates of \cref{alg:g-admm-cvx}}}]\label{lem:outloop-bd}
Suppose that Assumptions~\ref{asp:cvx-lclc} to \ref{asp:out-analysis} hold. Let $\{w^k\}_{k\ge0}$ be all the iterates generated by \cref{alg:g-admm-cvx}. Then we have $w^k\in\cQ_1$ for all $k\ge0$, where 
\begin{equation}\label{r0-theta}
\cQ_1 := \{w\in\bR^d:\|w-w^*\|\le r_0+2\sqrt{n}\bar{s}\beta\}\; \quad \text{with} \; r_0 := \|(w^0,\mu^0)-(w^*,\mu^*)\|,
\end{equation}
and $w^0$, $\mu^0$, $\bar{s}$, and $\beta$ are inputs of \cref{alg:g-admm-cvx}.  
\end{lemma}

This boundedness result allows us to utilize the Lipschitz continuity on a bounded set to establish the convergence rate for \cref{alg:g-admm-cvx}. The following theorem states the worst-case complexity results of \cref{alg:g-admm-cvx}, whose proof is relegated to \cref{apx:cplx-pal}.
\begin{theorem}[{{\bf complexity results of \cref{alg:g-admm-cvx}}}]\label{thm:outer-complex}
Suppose that Assumptions~\ref{asp:cvx-lclc} to \ref{asp:out-analysis} hold. Then,
\begin{enumerate}
\item[{\rm (a)}] the number of outer iteration of \cref{alg:g-admm-cvx} is at most $\cO(\max\{\epsilon_1^{-2},\epsilon_2^{-2}\})$; and
\item[{\rm (b)}] the total number of inner iterations of \cref{alg:g-admm-cvx} is at most $\widetilde{\cO}(\max\{\epsilon_1^{-2},\epsilon_2^{-2}\})$.
\end{enumerate}
\end{theorem}
\begin{remark} (a) To the best of our knowledge, \cref{thm:outer-complex} provides the first worst-case complexity results for finding an approximate optimal solution of \cref{FL-gl-cone} in an FL framework; (b) The number of outer and inner iterations of \cref{alg:g-admm-cvx} with detailed dependencies on the algorithm hyperparameters can be found in \cref{def:specific-total-outer,def:specific-total-inner} in the proofs, respectively.
\end{remark}

\subsection{Communication overheads}
In the outer loop of \cref{alg:g-admm-cvx}, a single communication round occurs after solving a proximal AL subproblem. During this round, the central server sends the current weights $w^{k+1}$ to all local clients, and each client sends back the maximum change in their respective multipliers, measured by $\|\mu^{k+1}_i-\mu^k_i\|_\infty$, to the central server. The communication overheads of the inner solver \cref{alg:admm-cvx-1} are discussed in \cref{sec:comm-inner}. The communication complexity of \cref{alg:g-admm-cvx} is $\widetilde{\mathcal{O}}(\max\{\epsilon_1^{-2},\epsilon_2^{-2}\})$.

\section{An inexact ADMM for FL}\label{sec:admm}
In this section, we propose an inexact ADMM-based FL algorithm to solve the subproblem in \cref{prox-AL-pb-compact} (the same as \cref{prox-AL-pb}) for \cref{alg:g-admm-cvx}. Before proceeding, we show that $\nabla P_{i,k}$, $0\le i\le n$, are locally Lipschitz continuous on $\bR^d$, whose proof is deferred to \cref{apx:lcl-lip}. 
\begin{lemma}[{{\bf local Lipschitz continuity of $\nabla P_{i,k}$}}]\label{lem:local-Lip}
Suppose that Assumptions~\ref{asp:cvx-lclc} to \ref{asp:out-analysis} hold. Then the gradients $\nabla P_{i,k}$, $0\le i\le n$, are locally Lipschitz continuous on $\bR^d$.
\end{lemma}

\begin{remark}\label{rmk:not-loc-lip}
It is worth noting that $\nabla P_{i,k}$, $0\le i\le n$, are typically not globally Lipschitz continuous on $\bR^d$ even if $\nabla f_i$, $1\le i\le n$, and $\nabla c_i$, $0\le i\le n$, are globally Lipschitz continuous on $\bR^d$. For example, consider $c_0(w)=\|w\|^2-1$. By \cref{def:P0k}, one has that 
\begin{equation*}
\nabla P_{0,k}(w) =  2[\mu_0^k+\beta(\|w\|^2-1)]_+ w + \frac{1}{(n+1)\beta}(w-w^k).
\end{equation*}
In this case, it is not hard to verify that $\nabla c_0$ is globally Lipschitz continuous on $\bR^d$, but $\nabla P_{0,k}$ is not. Thus, analyzing the complexity results for solving the subproblems in \cref{prox-AL-pb-compact} using local Lipschitz conditions of $\nabla P_{i,k}$, $0\le i\le n$, is reasonable.
\end{remark}
Moreover, it is easy to see that $P_{i,k}$ are strongly convex with the modulus $1/[(n+1)\beta]$ for all $0\le i\le n$ and all $k\ge 0$. 

Since both the local Lipschitz and the strong convexity (including its modulus) properties hold for all $k \ge 0$, and we need to solve the subproblem of the same form each $k$, below we drop $k$ and focus on solving the following model problem in an FL manner: 
\begin{equation}\label{pb:fs-gl}
\min_{w}  \left\{\ell(w):=\sum_{i=0}^n P_i(w; Z_i)  + h(w)\right\},
\end{equation}
where the data $Z_i$'s are only accessible to their corresponding local/global functions $P_i$'s, necessitating FL. We will drop $Z_i$'s henceforth for simplicity. The model problem in \cref{pb:fs-gl} satisfies: 
\begin{enumerate}[leftmargin=1em]
\item The functions $P_i$, $0\le i\le n$, are continuously differentiable, and moreover, $\nabla P_i$, $0\le i\le n$, are locally Lipschitz continuous on $\bR^d$; 
\item The functions $P_i$, $0\le i\le n$, are strongly convex with a modulus $\sigma> 0$ on $\bR^d$, that is,
\begin{equation}\label{Pi-s-convex}
\langle\nabla P_i(u)-\nabla P_i(v),u-v\rangle\ge\sigma \|u-v\|^2,\quad \forall u,v\in\bR^d,\ 0\le i\le n.
\end{equation}
\end{enumerate}


\subsection{Algorithm description}
\begin{algorithm}[h]
\caption{An inexact ADMM based FL algorithm for solving \cref{pb:fs-gl}}
\label{alg:admm-cvx-1}
\noindent\textbf{Input}: tolerance $\tau\in(0,1]$, $q\in(0,1)$, $\tilde{w}^0\in\dom(h)$, and $\rho_i>0$ for $1\le i\le n$;
\begin{algorithmic}[1]
\State Set $w^0=\tilde{w}^0$, and $(u_i^0,\lambda_i^0,\tilde{u}_i^0)=(\tilde{w}^0,-\nabla P_i(\tilde{w}^0),\tilde{w}^0-\nabla P_i(\tilde{w}^0)/\rho_i)$ for $1\le i\le n$.
\For{$t=0,1,2,\ldots$}
\State Set $\varepsilon_{t+1}=q^t$;\label{def:for-loop-beg}
\State {\bf Server update:} The central server finds an approximate solution $w^{t+1}$ to
\begin{equation}\label{sbpb:phi-gt}
\min_w\left\{\varphi_{0,t}(w):=P_0(w)+h(w)+\sum_{i=1}^n\left[\frac{\rho_i}{2}\|\tilde{u}_i^t-w\|^2\right]\right\}
\end{equation} 
\State such that $\rmdist_{\infty}(0,\partial \varphi_{0,t}(w^{t+1}))\le\varepsilon_{t+1}$.
\State {\bf Communication (broadcast):}  Each local client $i$, $1\le i\le n$, receives $w^{t+1}$ from the server. 
\State {\bf Client update (local):} Each local client $i$, $1\le i\le n$, finds an approximate solution $u_i^{t+1}$ to
\begin{equation}\label{sbpb:phi-it}
\min_{u_i}\left\{\varphi_{i,t}(u_i) := P_i(u_i) + \langle\lambda_i^t,u_i-w^{t+1}\rangle + \frac{\rho_i}{2}\|u_i-w^{t+1}\|^2\right\}
\end{equation}
\State such that $\|\nabla \varphi_{i,t}(u_i^{t+1})\|_{\infty}\le\varepsilon_{t+1}$, and then updates
\begin{align}
\lambda_i^{t+1}   = &\ \lambda_i^t + \rho_i(u_i^{t+1}-w^{t+1}),\label{sbpb:lam-gt}\\
\tilde{u}_i^{t+1} = &\ u_i^{t+1}+\lambda_i^{t+1}/\rho_i,\label{def:tuit}\\
\tilde{\varepsilon}_{i,t+1} = &\ \|\nabla\varphi_{i,t}(w^{t+1})-\rho_{i}(w^{t+1}-u_i^t)\|_{\infty}.\label{vareps-gt}
\end{align}
\State {\bf Communication:} Each local client $i$, $1\le i\le n$, sends $(\tilde{u}_i^{t+1},\tilde{\varepsilon}_{i,t+1})$ back to the central server.
\State {\bf Termination (server side):} Output $w^{t+1}$ and terminate this algorithm if
\begin{equation}\label{sbpb:stop}
\varepsilon_{t+1} + \sum_{i=1}^n\tilde\varepsilon_{i,t+1}\le \tau.
\end{equation}\label{def:for-loop-end}
\EndFor
\end{algorithmic}
\end{algorithm}

In this subsection, we propose an inexact ADMM-based FL algorithm (\cref{alg:admm-cvx-1}) for solving \cref{pb:fs-gl}. To make each participating client $i$ handle their local objective $P_i$ independently (see \cref{sec:alg-cvx-obj}), we introduce decoupling variables $u_i$'s and obtain the following equivalent consensus reformulation for \cref{pb:fs-gl}: 
\begin{equation}\label{pb:consens}
\min_{w,u_i} \left\{\sum_{i=1}^n P_i(u_i) + P_0(w) + h(w)\right\}\quad \st\quad u_i=w,\quad 1\le i\le n,
\end{equation}
which allows each local client $i$ to handle the local variable $u_i$ and the local objective function $P_i$ while imposing consensus constraints that force clients’ local parameters $u_i$ equal to the global parameter $w$. This reformulation enables the applicability of an inexact ADMM that solves \cref{pb:consens} in a federated manner. At each iteration, an ADMM solver optimizes the AL function associated with \cref{pb:consens}:
\begin{equation}\label{AL-consensus-F}
\cL_P(w,u,\lambda):=\sum_{i=1}^n \left[P_i(u_i)+\langle\lambda_i, u_i-w\rangle + \frac{\rho_i}{2}\|u_i-w\|^2\right] + P_0(w) + h(w)
\end{equation}
with respect to the variables $w$, $u$, and $\lambda$ alternately, where $u=[u_1^T,\ldots,u_n^T]^T$ and $[\lambda_1^T,\ldots,\lambda_n^T]^T$ collect all the local parameters and the multipliers associated with the consensus constraints, respectively. Specifically, in iteration $t$, one performs
\begin{align}
&w^{t+1}\approx\argmin_w \cL_P(w,u^t,\lambda^t),\label{w-admm}\\
&u^{t+1}\approx\argmin_{u} \cL_P(w^{t+1},u,\lambda^t),\label{u-admm}\\
&\lambda_i^{t+1} = \lambda_i^t + \rho_i (u_i^{t+1}-w^{t+1}),\quad \forall 1\le i\le n.
\end{align}
By the definition of $\cL_P$ in \cref{AL-consensus-F}, one can verify that the step in \cref{w-admm} is equivalent to \cref{sbpb:phi-gt}, and also the step in \cref{u-admm} can be computed in parallel, which corresponds to \cref{sbpb:phi-it}. Therefore, the ADMM updates naturally suit the FL framework, as the separable structure in \cref{AL-consensus-F} over the pairs $\{(u_i,\lambda_i)\}$ enables the local update of $(u_i,\lambda_i)$ at each client $i$ while $w$ is updated by the central server.

Since the subproblems in \cref{sbpb:phi-gt} and \cref{sbpb:phi-it} are strongly convex, their approximate solutions $w^{t+1}$ and $u_i^{t+1}$, $1\le i\le n$, can be found using a gradient-based algorithm with a global linear convergence rate~\cite{nesterov2018lectures}. Furthermore, the value $\tilde{\varepsilon}_{i,t+1}$ in \cref{vareps-gt} serves as a measure of local optimality and consensus for client $i$. By summing up $\tilde{\varepsilon}_{i,t+1}$ for $1\le i\le n$ and including $\varepsilon_{t+1}$, one can obtain a stationarity measure for the current iterate (see (\cref{sbpb:stop})), as presented in the following theorem. Its proof can be found in \cref{apx:out-admm}.

\begin{remark}\label{rmk:alg2}
    We now make the following remarks on \cref{alg:admm-cvx-1}.
    \begin{enumerate}[label=(\alph*),leftmargin=*]
\vspace{-0.5em}
    \item On hyperparameters of \cref{alg:admm-cvx-1},
\begin{itemize}[leftmargin=*]
    \item $(\tau,\Tilde{w}^0)$ is specified as $(\tau_k,w^k)$ at the $k$th iteration of \cref{alg:g-admm-cvx}.
    \item From \cref{sbpb:phi-gt}, $\rho_i$, $1\le i\le n$ can be viewed as weighting parameters for aggregation. Therefore, it is natural to set $\rho_i = a m_i$ for $1 \leq i \leq n$, where $m_i$ is the number of samples in client $i$ and $a$ is a global constant. We follow this rule when setting $\rho_i$'s.
    \item $q\in(0,1)$ determines the tolerance sequence $\{\varepsilon_{t+1}\}_{t\ge0}$ for the subproblems in \cref{sbpb:phi-gt}. These tolerances in solving subproblems reduce computational costs. Setting $\{\varepsilon_{t+1}\}_{t\ge0}$ to rapidly diminish toward zero rapidly at a geometric rate ensures the convergence of \cref{alg:g-admm-cvx}. In practice, we suggest setting $q$ as $\mathcal{O}(1)$.
\end{itemize}
    \item The main innovations we have here compared to the existing literature on ADMM-based FL algorithms (e.g, \cite{ZL23ADMM,GLF22fedadmm,ZHDYL21fedpd}) include:
    \begin{itemize}[leftmargin=*]
    \item We establish the complexity results of an inexact ADMM-based FL algorithm under local Lipschitz conditions, vs. global Lipschitz conditions in other work. Our complexity results can be found in \cref{thm:complexity-alg1}.
    \item We propose a novel and rigorous stopping criterion (\cref{sbpb:stop}) that is easily verifiable, communication-light, and compatible with the outer iterations (as our inexact ADMM FL algorithm serves as a subproblem solver in our overall algorithm framework). 
    \end{itemize} 
    \end{enumerate}
\end{remark}

\begin{theorem}[{{\bf output of \cref{alg:admm-cvx-1}}}]\label{thm:output-alg1}
If \cref{alg:admm-cvx-1} terminates at some iteration $T \ge 0$, then its output $w^{T+1}$ satisfies $\rmdist_{\infty}(0,\partial \ell(w^{T+1}))\le\tau$.
\end{theorem}
\cref{thm:output-alg1} states that \cref{alg:admm-cvx-1} outputs a point that approximately satisfies the first-order optimality condition of \cref{FL-gl-cone}. In addition, it follows from \cref{thm:output-alg1} that \cref{alg:admm-cvx-1} with $(\tau,\tilde{w}^0)=(\tau_k,w^k)$ finds an approximate solution $w^{k+1}$ to \cref{prox-AL-pb-compact} such that \cref{cond:apx-stat} holds.

\subsection{Complexity analysis}
In this subsection, we establish the iteration complexity for the inexact ADMM, namely, \cref{alg:admm-cvx-1}. Recall from \cref{Pi-s-convex} that \cref{pb:fs-gl} is strongly convex and thus has a unique optimal solution. We refer to this optimal solution of \cref{pb:fs-gl} as $\tilde{w}^*$ throughout this section. The following lemma shows that all the iterates generated by \cref{alg:admm-cvx-1} lie in a compact set. Its proof can be found in \cref{apx:bd-admm}.
\begin{lemma}[{{\bf bounded iterates of \cref{alg:admm-cvx-1}}}]\label{lem:uit-wt-bd}
Suppose that Assumptions~\ref{asp:cvx-lclc} to \ref{asp:out-analysis} hold and let $\{u_i^{t+1}\}_{1\le i\le n,t\ge0}$ and $\{w^{t+1}\}_{t\ge0}$ be all the iterates generated by \cref{alg:admm-cvx-1}. Then it holds that all these iterates stay in a compact set $\cQ$, where 
\begin{equation}\label{def:Q_Fh}
\cQ := \left\{v:\|v-\tilde{w}^*\|^2\le\frac{n+1}{\sigma^2(1-q^2)}+\frac{1}{\sigma} \sum_{i=1}^n\left(\rho_i\|\tilde{w}^* - \tilde{w}^0\|^2 + \frac{1}{\rho_i} \|\nabla P_i(\tilde{w}^*) - \nabla P_i(\tilde{w}^0)\|^2\right)\right\}.
\end{equation}
\end{lemma}

The iteration complexity of \cref{alg:admm-cvx-1} is established in the following theorem, whose proof is relegated to \cref{apx:cplx-admm}. 
\begin{theorem}[{{\bf iteration complexity of \cref{alg:admm-cvx-1}}}]\label{thm:complexity-alg1}
Suppose that Assumptions~\ref{asp:cvx-lclc} to \ref{asp:out-analysis} hold. Then \cref{alg:admm-cvx-1} terminates in at most $\cO(|\log\tau|)$ iterations. 
\end{theorem}

\begin{remark} We now make the following remarks on the complexity results in \cref{thm:complexity-alg1}.
\begin{enumerate}[label=(\alph*),leftmargin=*]
\vspace{-0.5em}
    \item \cref{alg:admm-cvx-1} enjoys a global linear convergence rate when solving the problem in \cref{pb:fs-gl}. The result generalizes classical convergence results for ADMM in the literature, which typically require a strongly convex objective with {\it globally} Lipschitz continuous gradient (e.g., see \cite{lin2015global}). In contrast, our result is the first to establish a global linear convergence of an inexact ADMM assuming a strongly convex objective with only a {\it locally} Lipschitz continuous gradient.  
    \item The number of iterations of \cref{alg:admm-cvx-1} with dependencies on all the algorithm hyperparameters can be found in \cref{iter-alg1} in the proofs.

    \item The general research on complexity analysis for optimization algorithms under local Lipschitz assumptions is relatively new. For example, \cite{LM22} proposes accelerated gradient methods for convex optimization problems with locally Lipschitz continuous gradients, and \cite{zhang2024first} proposes accelerated gradient methods for nonconvex optimization problems with locally Lipschitz continuous gradients.
\end{enumerate}
\end{remark}

\subsection{Communication overheads}\label{sec:comm-inner}
In each iteration of \cref{alg:admm-cvx-1}, a single communication round happens between the clients and the central server. During this round, the central server transmits the global weight $w^{t+1}$ to all clients, and subsequently each local client performs multiple local updates to solve a local subproblem and then sends the updated local weights $\tilde{u}_i^{t+1}$ and a local stationarity measure $\tilde{\varepsilon}_{i,t+1}$ back to the central server. The communication complexity of each call of \cref{alg:admm-cvx-1} is $\cO(|\log\tau|)$.

\section{Numerical experiments}
\label{sec:exp}
Here, we conduct numerical experiments to evaluate the performance of our proposed FL algorithm (\cref{alg:g-admm-cvx}). Specifically, we benchmark our algorithm against a centralized proximal AL method (cProx-AL, described in \cref{alg:c-prox-AL}) on a convex Neyman-Pearson classification problem (\cref{sec:npc}) and a fair-aware learning problem (\cref{sec:fair-class}) with real-world datasets, and further on linear-equality-constrained quadratic programming problems with simulated data (\cref{sec:lcqp}).
All experiments are carried out on a Windows system with an AMD EPYC 7763 64-core processor, and all algorithms are implemented in Python. The code to implement the proposed algorithm on these numerical examples is available at \url{https://github.com/PL97/Constr_FL}.

\begin{table}[t]
\centering
\caption{Numerical results for solving \cref{NP-class} using our algorithm vs. using cProx-AL. Inside the parentheses are the respective standard deviations over 10 random trials. For feasibility, we include the mean and maximum losses for class 1 among all local clients.}
\smallskip
\resizebox{\textwidth}{!}{
\begin{tabular}{c|c||ccc||cccc}
\hline
\multirow{3}{*}{dataset} & \multirow{3}{*}{$n$} & \multicolumn{3}{c||}{objective value (loss for class 0)} &  \multicolumn{4}{c}{feasibility (loss for class 1 ($\le0.2$))} \\
& & \cref{alg:g-admm-cvx} & cProx-AL &  relative difference  & \multicolumn{2}{c}{\cref{alg:g-admm-cvx}} & \multicolumn{2}{c}{cProx-AL} \\
&& & & & mean & max & mean & max \\ 

\hline
 \multirow{4}{*}{breast-cancer-wisc}& 1  & 0.27 (1.52e-04)  & 0.27 (3.02e-05)  & 7.09e-04 (2.02e-04)  & 0.20 (1.80e-07)  & 0.20 (1.80e-07)  & 0.20 (1.84e-08)  & 0.20 (1.84e-08)  \\
 & 5  & 0.34 (4.50e-02)  & 0.33 (4.55e-02)  & 1.15e-02 (5.17e-03)  & 0.19 (7.33e-06)  & 0.20 (1.08e-06)  & 0.19 (1.13e-04)  & 0.20 (1.72e-05)  \\
 & 10  & 0.37 (1.08e-01)  & 0.37 (1.08e-01)  & 3.92e-04 (2.76e-04)  & 0.17 (1.15e-05)  & 0.20 (6.05e-09)  & 0.17 (1.14e-05)  & 0.20 (2.95e-08)  \\
 & 20  & 0.46 (2.12e-01)  & 0.45 (2.12e-01)  & 3.43e-02 (2.91e-02)  & 0.16 (3.52e-05)  & 0.20 (3.76e-06)  & 0.16 (7.03e-06)  & 0.20 (7.70e-08)  \\
\hline
 \multirow{4}{*}{adult-a}& 1  & 0.73 (2.19e-04)  & 0.73 (1.25e-04)  & 2.24e-04 (3.46e-04)  & 0.20 (6.30e-07)  & 0.20 (6.30e-07)  & 0.20 (1.73e-06)  & 0.20 (1.73e-06)  \\
 & 5  & 0.74 (1.03e-02)  & 0.74 (1.03e-02)  & 4.25e-03 (7.44e-04)  & 0.20 (2.14e-04)  & 0.20 (2.80e-04)  & 0.20 (1.21e-05)  & 0.20 (2.28e-06)  \\
 & 10  & 0.77 (1.98e-02)  & 0.77 (1.98e-02)  & 2.69e-03 (3.24e-03)  & 0.19 (6.41e-05)  & 0.20 (9.76e-05)  & 0.19 (2.00e-05)  & 0.20 (1.23e-05)  \\
 & 20  & 0.78 (2.86e-02)  & 0.79 (2.81e-02)  & 1.13e-02 (4.11e-03)  & 0.18 (6.40e-04)  & 0.20 (6.59e-05)  & 0.18 (1.96e-05)  & 0.20 (3.19e-06)  \\
\hline
 \multirow{4}{*}{monks-1}& 1  & 1.58 (7.61e-05)  & 1.58 (7.50e-05)  & 1.39e-05 (1.09e-05)  & 0.20 (1.09e-07)  & 0.20 (1.09e-07)  & 0.20 (3.01e-07)  & 0.20 (3.01e-07)  \\
 & 5  & 1.65 (8.39e-02)  & 1.65 (8.41e-02)  & 2.08e-04 (1.84e-04)  & 0.19 (6.39e-05)  & 0.20 (5.39e-05)  & 0.19 (5.04e-06)  & 0.20 (5.60e-07)  \\
 & 10  & 1.71 (1.18e-01)  & 1.71 (1.18e-01)  & 4.59e-04 (3.32e-04)  & 0.18 (3.98e-05)  & 0.20 (4.46e-05)  & 0.18 (6.44e-06)  & 0.20 (1.60e-06)  \\
 & 20  & 1.81 (1.49e-01)  & 1.79 (1.60e-01)  & 1.78e-02 (1.38e-02)  & 0.17 (1.68e-04)  & 0.20 (2.24e-04)  & 0.17 (4.60e-06)  & 0.20 (1.62e-06)  \\
\hline 
\end{tabular}
}
\label{table:FL-np}
\end{table}
\begin{figure}[!htbp]
\centering
\includegraphics[width=0.95\linewidth]{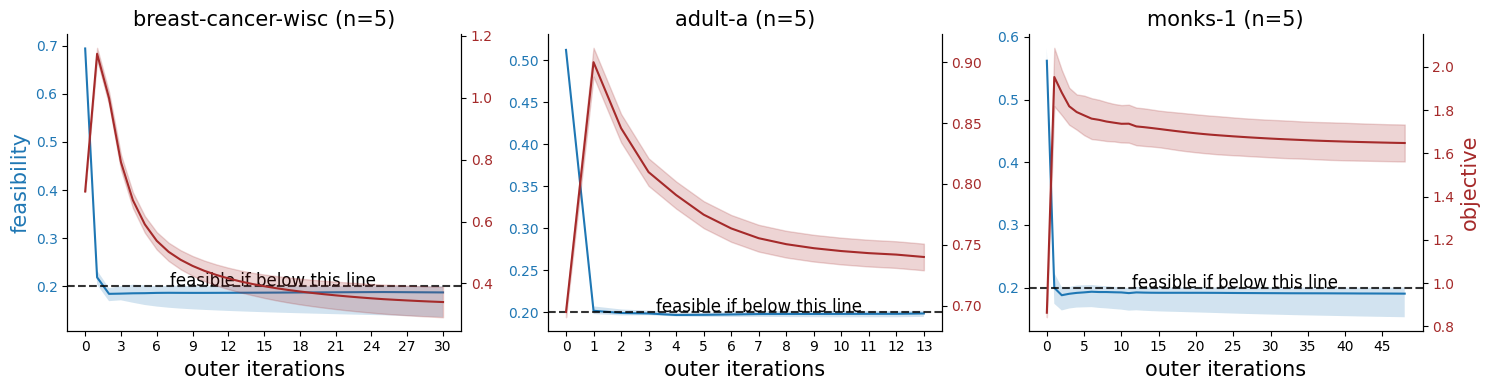}
\caption{Convergence behavior of local objective and local feasibility in one random trial over the outer iterations of \cref{alg:g-admm-cvx} on three real-world datasets. The solid blue and brown lines indicate the mean local objective and the mean local feasibility over all clients, respectively. The blue and the brown areas indicate the cross-client variations of local objectives and local feasibility, respectively. The dashed black line indicates the feasibility threshold.
}
\label{fig:feas_progression}
\end{figure}

\subsection{Neyman-Pearson classification}\label{sec:npc}
In this subsection, we consider the Neyman-Pearson classification problem:
\begin{equation}\label{NP-class}
\min_{w}\frac{1}{n}\sum_{i=1}^{n}\frac{1}{m_{i0}}\sum_{j=1}^{m_{i0}}\phi(w;(x_{j}^{(i0)},0))\quad \st\quad \frac{1}{m_{i1}}\sum_{j=1}^{m_{i1}}\phi(w;(x_{j}^{(i1)},1))\le r_i,\quad 1\le i\le n,
\end{equation}
where $\{x_{j}^{(i0)}\}_{1\le j\le m_{i0}}$ and $\{x_{j}^{(i1)}\}_{1\le j\le m_{i1}}$ are the sets of samples at client $i$ associated with labels $0$ and $1$, respectively, and $\phi$ is the binary logistic loss \cite{hastie2009elements} 
\begin{equation}\label{def:logistic-loss}
\phi(w;(x,y))=-y w^Tx +\log (1+e^{w^Tx}),\quad y\in\{0,1\}.
\end{equation}
Then, both the objective and the constraints in \cref{NP-class} are convex. We consider three real-world datasets, namely `breast-cancer-wisc', `adult-a', and `monks-1', from the UCI repository\footnote{see \url{https://archive.ics.uci.edu/datasets}} and described in \cref{apx:dataset-desc}. For each dataset, we perform the Neyman-Pearson classification that minimizes the loss of classification for class 0 (majority) while ensuring that the loss for class 1 (minority) is less than a threshold $r_i=0.2$. To simulate the FL setting, we divide each dataset into $n$ folds, mimicking local clients, each holding the same amount of data with equal ratios of the two classes.

We apply \cref{alg:g-admm-cvx} and cProx-AL (\cref{alg:c-prox-AL}) to find a $(10^{-3},10^{-3})$-optimal solution of \cref{NP-class}. We run 10 trials of \cref{alg:g-admm-cvx} and cProx-AL. For each run, both algorithms have the same initial point $w^0$, randomly chosen from the unit Euclidean sphere. We set the other parameters for \cref{alg:g-admm-cvx} and cProx-AL as $\mu^0_i=(0,\ldots,0)^T\ \forall 0\le i\le n$, $\bar{s}=0.001$ and $\beta=300$. We also set $\rho_i=0.01\ \forall 1\le i\le n$ for \cref{alg:admm-cvx-1}.


Comparing the objective value and feasibility of solutions achieved by \cref{alg:g-admm-cvx} and cProx-AL in \cref{table:FL-np}, we see that both algorithms can yield solutions of similar quality. Given the small standard deviations, we observe that the convergence behavior of \cref{alg:g-admm-cvx} remains stable across 10 trial runs. These observations demonstrate the ability of \cref{alg:g-admm-cvx} to reliably solve \cref{NP-class} in the FL setting without compromising solution quality. 
From \cref{fig:feas_progression}, we observe that \cref{alg:g-admm-cvx} consistently achieves feasibility for all local constraints while also minimizing all the local objectives.

\subsection{Classification with fairness constraints}\label{sec:fair-class}

In this subsection, we consider fairness-aware learning with global and local fairness constraints:
\begin{equation}\label{fair-class}
    \min_{w} \frac{1}{n}\sum_{i=1}^{n}\frac{1}{m_i}\sum_{j=1}^{m_i}\phi(w;z_j^{(i)})\quad \st \quad -r_i\le\frac{1}{\tilde{m}_i}\sum_{j=1}^{\tilde{m}_i}\phi(w;\tilde{z}_j^{(i)}) - \frac{1}{\hat{m}_i}\sum_{j=1}^{\hat{m}_i}\phi(w;\hat{z}_j^{(i)})\le r_i,\quad 0\le i\le n.    
\end{equation}
Here, $\{z_j^{(i)}=(x_{j}^{(i)},y_j^{(i)})\in\bR^d\times\{0,1\}: i=0,\ldots,n, j=1,\ldots,m_i\}$ is the training set, where $i$ indexes the central server/local clients. For each $i=0,\ldots,n$, the dataset $\{z_j^{(i)}\}_{1\le j\le m_i}$ is further divided into two subgroups $\{\tilde{z}_{j}^{(i)}\}_{1\le j\le \tilde{m}_{i}}$ and $\{\hat{z}_{j}^{(i)}\}_{1\le j\le \hat{m}_{i}}$ based on certain subgroup attributes. The constraints with $i=1,\ldots,n$ refer to local constraints at client $i$, while the constraints with $i=0$ refer to global constraints at the central server. 

We choose $\phi$ as the binary logistic loss defined in \cref{def:logistic-loss}, leading to nonconvex constraints in \cref{fair-class}. For the real-world dataset, we consider `adult-b'\footnote{This dataset can be found in \url{https://github.com/heyaudace/ml-bias-fairness/tree/master/data/adult}.}: each sample in this dataset has $39$ features and one binary label. To simulate the FL setting, we divide the $22,654$ training samples from the `adult-b' dataset into $n$ folds and distribute them to $n$ local clients. The central server holds the $5,659$ test samples from the `adult-b' dataset. Note that although we have taken both the ``training'' and ``test'' samples from the `adult-b' dataset here, these samples are used to simulate our local samples and central samples, respectively. The focus here is to test optimization performance, not generalization---we do not have a test step, unlike in typical supervised learning. 


We apply \cref{alg:g-admm-cvx} and cProx-AL (\cref{alg:c-prox-AL}) to find a $(10^{-3},10^{-3})$-optimal solution of \cref{fair-class}. 
We run 10 trials of \cref{alg:g-admm-cvx} and cProx-AL. For each run, both algorithms have the same initial point $w^0$, randomly chosen from the unit Euclidean sphere. We set the other parameters for \cref{alg:g-admm-cvx} and cProx-AL as $\mu^0_i=(0,\ldots,0)^T\ \forall 0\le i\le n$, $\bar{s}=0.001$ and $\beta=10$. We also set $\rho_i=10^8\ \forall 1\le i\le n$ for \cref{alg:admm-cvx-1}.


Comparing the objective value and feasibility of solutions achieved by \cref{alg:g-admm-cvx} and cProx-AL in \cref{table:FL-fair-g} reveals that \cref{alg:g-admm-cvx} and cProx-AL can produce solutions of similar quality. Given the small standard deviations, we observe that the convergence behavior of \cref{alg:g-admm-cvx} remains stable across 10 trial runs. These observations demonstrate the ability of \cref{alg:g-admm-cvx} to reliably solve \cref{fair-class} in the FL setting without compromising solution quality. It also suggests the potential of our algorithm in solving FL problems with nonconvex constraints. 
From \cref{fig:fair_progression}, we see that our proposed method consistently achieves feasibility for all local and global constraints while also minimizing all the local objectives.

\begin{table}
\centering
\caption{Numerical results for \cref{fair-class} using our algorithm vs. using cProx-AL. Inside the parentheses are the respective standard deviations over 10 random trials. For feasibility, we include the mean and maximum loss disparities (absolute difference between losses for two subgroups) among all clients and the central server.}
\smallskip
\resizebox{\textwidth}{!}{
\begin{tabular}{c||ccc||cccc}

\hline
\multirow{3}{*}{$n$} & \multicolumn{3}{c||}{objective value} &  \multicolumn{4}{c}{feasibility (loss disparity ($\le0.1$))} \\
& \cref{alg:g-admm-cvx} & cProx-AL & relative difference & \multicolumn{2}{c}{ \cref{alg:g-admm-cvx}} & \multicolumn{2}{c}{cProx-AL} \\
&  &  & & mean & max & mean & max \\\hline
 1  & 0.37 (9.83e-05) & 0.37 (4.14e-05) & 1.97e-03 (2.53e-04) &  0.10 (1.14e-04) & 0.10 (1.36e-04) & 0.10 (3.69e-06) & 0.10 (5.38e-06) \\
 5  & 0.37 (3.99e-03) & 0.37 (4.05e-03) & 1.86e-03 (4.69e-04) &  0.09 (5.34e-05) & 0.10 (7.51e-05) & 0.09 (3.68e-05) & 0.10 (4.36e-06) \\
 10 & 0.37 (6.39e-03) & 0.37 (6.52e-03) & 2.39e-03 (8.40e-04) &  0.08 (1.68e-04) & 0.10 (2.15e-05) & 0.08 (1.52e-04) & 0.10 (6.56e-06) \\
 20 & 0.38 (9.46e-03) & 0.37 (9.86e-03) & 4.61e-03 (2.43e-03) &  0.08 (9.75e-05) & 0.10 (1.01e-04) & 0.08 (4.90e-05) & 0.10 (6.06e-06) \\
\hline
\end{tabular}
}
\label{table:FL-fair-g}
\end{table}

\begin{figure}[!htbp]
\centering
\includegraphics[width=0.95\linewidth]{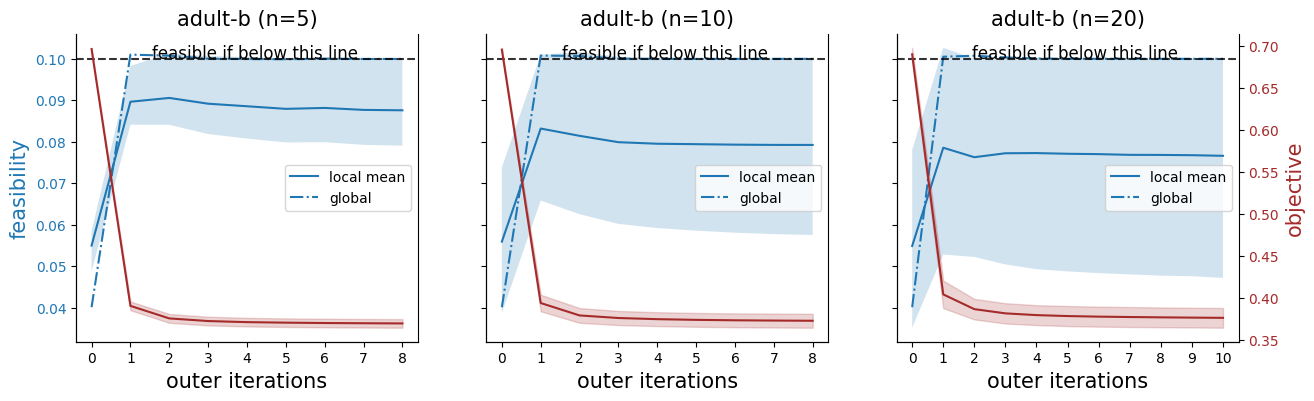}
\caption{Convergence of local objective, local feasibility, and the feasibility for global constraints in one random trial over the outer iterations of \cref{alg:g-admm-cvx}. The solid blue and brown lines indicate the mean local objective and the mean local feasibility over all clients, respectively. The blue and brown areas indicate the cross-client variations of local objectives and local feasibility, respectively. The dashdot blue line indicates the feasibility for global constraints. The dashed black line indicates the feasibility threshold.}
\label{fig:fair_progression}
\end{figure}

\section{Concluding remarks}
In this paper, we propose an FL algorithm for solving general constrained ML problems based on the proximal AL method. We analyze the worst-case iteration complexity of the proposed algorithm, assuming convex objective and convex constraints with locally Lipschitz continuous gradients. Finally, we perform numerical experiments to assess the performance of the proposed algorithm for constrained classification problems, using real-world datasets. The numerical results clearly demonstrate the practical efficacy of our proposed algorithm. Since our work is the first of its kind, there are numerous possible future directions. For example, one could try to extend our FL algorithms to allow partial client participation and stochastic solvers at local clients. In addition, developing FL algorithms for general constrained ML with convergence guarantees in nonconvex settings remains largely open. Lastly, constrained FL with a fixed iteration and communication budget, especially stringent ones, is a very useful but challenging future research topic.

\subsubsection*{Acknowledgments}
C. He is partially supported by the NIH fund R01CA287413 and the UMN Research Computing Seed Grant. L. Peng is partially supported by the CISCO Research fund 1085646 PO USA000EP390223. J. Sun is partially supported by the NIH fund R01CA287413 and the CISCO Research fund 1085646 PO USA000EP390223. The authors acknowledge the Minnesota Supercomputing Institute (MSI) at the University of Minnesota for providing resources that contributed to the research results reported in this article. The content is solely the responsibility of the authors and does not necessarily represent the official views of the National Institutes of Health. 

\bibliographystyle{plain}
\bibliography{main}

\appendix
\section*{Appendix}

In \cref{apx:mi-clean,apx:pf-admm,apx:pf-pal}, we provide proofs of the main results in \cref{sec:admm,sec:d-pal}. \cref{apx:cpal} presents a proximal AL method for centralized constrained optimization. In \cref{apx:ad-nr}, we include some extra numerical results.


\section{Proofs of \cref{thm:output-pal}, \cref{lem:outloop-bd}, and \cref{thm:outer-complex}(a)}\label{apx:mi-clean}

First, we set up the technical tools necessary for the proof, following \cite{LZ18AL}. With the abbreviations in \cref{abbre}, we define the Lagrangian function associated with \cref{FL-gl-cone,FL-cons-dual} as
\[
l(w,\mu) = \left\{\begin{array}{ll}
f(w) + h(w) + \langle\mu, c(w)\rangle& \text{if } w\in\dom(h)\text{ and }\mu\ge0, \\
-\infty&\text{if } w\in\dom(h)\text{ and }\mu\not\ge0,\\
\infty&\text{if } w\not\in\dom(h),
\end{array}\right.
\]
Then, one can verify that
\begin{equation}\label{def:partial-l}
\partial l(w,\mu)=\left\{\begin{array}{ll}
\begin{pmatrix}
\nabla f(w) +\partial h(w) + \nabla c(w)\mu\\
c(w) - \cN_{\bR^m_+}(\mu)
\end{pmatrix}&\text{if } w\in\dom(h) \text{ and } \mu\ge0,\\
\emptyset&\text{otherwise.}
\end{array}\right.    
\end{equation}
We also define a set-valued operator $\cT$ associated with \cref{FL-gl-cone,FL-cons-dual}:
\begin{equation}
\cT:(w,\mu)\to \{(u,\nu)\in\bR^d\times\bR^m:(u,-\nu)\in\partial l(w,\mu)\},\quad\forall (w,\mu)\in\bR^d\times\bR^m,\label{Tlwmu}
\end{equation}
which is maximally monotone (see, e.g., Section 2 of \cite{rockafellar1976augmented}). Finding a KKT solution of \cref{FL-gl-cone} can be viewed as solving the monotone inclusion problem~\cite{rockafellar1976augmented}:
\begin{equation}\label{MI}
\text{Find}\quad (w,\mu)\in\bR^d\times\bR^m\quad {\text{such that}} \quad (0,0)\in\cT(w,\mu).
\end{equation}
Furthermore, applying the proximal AL method to solve \cref{FL-gl-cone} is equivalent to applying the proximal point algorithm (PPA) to solve this monotone inclusion problem \cite{rockafellar1976augmented,mi2017}, that is, 
\begin{align} \label{eq:proxal_ppa_equiv}
&w^{k+1}=\argmin_{w}\ell_k(w),\quad \mu^{k+1} = [\mu^k + \beta c(w^{k+1})]_+,\quad \Longleftrightarrow \quad (w^{k+1},\mu^{k+1})= \cJ(w^k,\mu^k),\quad \forall k\ge0,
\end{align}
where $(w^0,\mu^0)\in\dom(h)\times \bR^m_+$ and $\cJ$ is the resolvent of $\cT$ defined as
\begin{equation}\label{reso}
\cJ := (\cI+\beta \cT)^{-1}    
\end{equation} 
with $\cI$ being the identity operator. When the $\argmin_{w}\ell_k(w)$ subproblem is only solved up to approximate stationarity, that is, $\mathrm{dist}_{\infty}(0, \partial \ell_k(w^{k+1})) \le \tau_k$ as in our \cref{alg:g-admm-cvx}, the error $\tau_k$ will propagate to the next iterate that we obtain. This is quantitatively captured by the following result. 



\begin{lemma}[{{\bf adaptation of Lemma~5 of \cite{LZ18AL}}}]\label{lem:tech1}
Suppose that Assumptions~\ref{asp:cvx-lclc} to \ref{asp:out-analysis} hold. Let $\{(w^k,\mu^k)\}_{k\ge0}$ be generated by \cref{alg:g-admm-cvx}. Then for any $k\ge0$, we have
\begin{equation*}
\|(w^{k+1},\mu^{k+1})-\cJ(w^k,\mu^k)\|\le\beta \sqrt{n}\tau_k,
\end{equation*}
where $\cJ$ is the resolvent of $\cT$ defined in \cref{reso}.
\end{lemma}

\begin{proof}
Notice from \cref{cond:apx-stat} that $\mathrm{dist}(0,\partial \ell_k(w^{k+1})) \le \sqrt{n}\mathrm{dist}_\infty(0,\partial \ell_k(w^{k+1})) \le \sqrt{n}\tau_k$. By this and Lemma~5 of \cite{LZ18AL}, the conclusion of this lemma holds. 
\end{proof}

\subsection{Proof of \cref{thm:output-pal}}\label{apx:out-pal}

\begin{proof}[Proof of \cref{thm:output-pal}]
Notice from \cref{abbre,prox-AL-pb} that
\[
\ell_k(w) = f(w) + h(w) + \frac{1}{2\beta}\left(\|[\mu^k+\beta c(w)]_+\|^2-\|\mu^k\|^2\right) + \frac{1}{2\beta}\|w-w^k\|^2.
\]
By this, \cref{def:partial-l}, and the fact that $\mu^{k+1}=[\mu^k+\beta c(w^{k+1})]_+$, one has
\begin{align}
\partial\ell_k(w^{k+1}) - \frac{1}{\beta} (w^{k+1}-w^k) =&\ \nabla f(w^{k+1}) + \partial h(w^{k+1}) + \nabla c(w^{k+1})[\mu^k+\beta c(w^{k+1})]_+\nonumber\\
=&\ \nabla f(w^{k+1}) + \partial h(w^{k+1}) + \nabla c(w^{k+1})\mu^{k+1} = \partial_w l(w^{k+1},\mu^{k+1}).\label{diff-w}
\end{align}
Notice that 
\[
\mu^{k+1}=[\mu^k+\beta c(w^{k+1})]_+=\argmin_{\mu\in\bR^m_+} \frac{1}{2}\|\mu - (\mu^k+\beta c(w^{k+1}))\|^2.
\]
By the optimality condition of this projection, we have 
\[
0 \in \mu^{k+1}-(\mu^k+\beta c(w^{k+1})) + \cN_{\bR^m_+}(\mu^{k+1}),
\]
which together with \cref{def:partial-l} implies that
\begin{equation}\label{diff-mu}
\frac{1}{\beta}(\mu^{k+1}-\mu^k)\in\partial_\mu l(w^{k+1},\mu^{k+1}).
\end{equation}
In view of this, \cref{cond:apx-stat,stop-alg,diff-w}, we can see that
\begin{align*}
\rmdist_{\infty}(0,\partial_w l(w^{k+1},\mu^{k+1}))&\overset{\cref{diff-w}}{\le} \rmdist_{\infty}(0,\partial\ell_k(w^{k+1})) +  \frac{1}{\beta}\|w^{k+1}-w^k\|_{\infty}\\
&\overset{\cref{cond:apx-stat}}{\le}\tau_k +  \frac{1}{\beta}\|w^{k+1}-w^k\|_{\infty}\overset{\cref{stop-alg}}{\le}\epsilon_1,\\
\rmdist_{\infty}(0,\partial_{\mu} l(w^{k+1},\mu^{k+1}))&\overset{\cref{diff-mu}}{\le} \frac{1}{\beta}\|\mu^{k+1}-\mu^k\|_{\infty}\overset{\cref{stop-alg}}{\le}\epsilon_2.
\end{align*}    
These along with \cref{def:partial-l} and \cref{def:eps-KKT} imply that $(w^{k+1},\mu^{k+1})$ is an $(\epsilon_1,\epsilon_2)$-KKT solution of \cref{FL-gl-cone}, which proves this theorem as desired.
\end{proof}

\subsection{Proof of \cref{lem:outloop-bd}}\label{apx:bd-pal}
Define
\begin{equation}\label{w-star-k}
w_*^k := \argmin_{w} \ell_k(w),\qquad \mu_*^k := [\mu^k+\beta c(w^k_*)]_+,\quad \forall k\ge0,  
\end{equation}
which, by \cref{eq:proxal_ppa_equiv}, is equivalent to 
\begin{equation}\label{Jkstar}
(w^k_*,\mu^k_*) = \cJ(w^k,\mu^k).    
\end{equation}
Recall that $(w^*, \mu^*)$ is assumed to be any pair of optimal solutions to \cref{FL-gl-cone} and \cref{FL-cons-dual}. Toward the proof, we first present an intermediate result, which mostly follows the fact that $\cJ$ is firmly nonexpansive. 
\begin{lemma}\label{lem:tech2}
Suppose that Assumptions \ref{asp:cvx-lclc} to \ref{asp:out-analysis} hold. Let $\{(w^k,\mu^k)\}_{k\ge0}$ be generated by \cref{alg:g-admm-cvx}. Let $(w^k_*,\mu^k_*)$ be defined in \cref{w-star-k} for all $k\ge0$. Then the following relations hold.
\begin{align}
&\|(w^k,\mu^k)-(w^k_*,\mu^k_*)\|^2+\|(w^k_*,\mu^k_*)-(w^*,\mu^*)\|^2\le\|(w^k,\mu^k)-(w^*,\mu^*)\|^2,\quad \forall k\ge0,\label{firm-nonexp}\\
&\|(w^k,\mu^k)-(w^*,\mu^*)\|\le \|(w^0,\mu^0)-(w^*,\mu^*)\|+ \beta\sqrt{n}\sum_{j=0}^{k-1}\tau_j,\quad\forall k\ge0.\label{seq-bd}
\end{align}
\end{lemma}

\begin{proof}
Since $(w^*,\mu^*)$ is a solution to the monotone inclusion problem~\cref{MI}, we have  
\begin{equation}\label{Jstar}
(0,0)\in\cT(w^*,\mu^*),\quad \text{and} \quad (w^*,\mu^*) = \cJ(w^*,\mu^*).
\end{equation}
Moreover, since $\cT$ is maximally monotone, its resolvent $\cJ$ is firmly nonexpansive (see, e.g., Corollary 23.9 of \cite{mi2017}), that is, $\|\cJ(w, \mu) - \cJ(w', \mu')\|^2 + \|(\cI-\cJ)(w, \mu) - (\cI-\cJ)(w', \mu')\|^2 \le \| (w, \mu) - (w', \mu') \|^2$ for any feasible pairs $(w, \mu)$ and $(w', \mu')$.  Using \cref{Jkstar,Jstar}, we obtain that 
\begin{align*}
&\|(w^k,\mu^k)-(w^k_*,\mu^k_*)\|^2+\|(w^k_*,\mu^k_*)-(w^*,\mu^*)\|^2\\      
&\overset{\cref{Jkstar,Jstar}}{=}\|(\cI-\cJ)(w^k,\mu^k)-(\cI-\cJ)(w^*,\mu^*)\|^2 + \|\cJ(w^k,\mu^k) - \cJ(w^*,\mu^*)\|^2\\ &\le \|(w^k,\mu^k)-(w^*,\mu^*)\|^2.     \quad (\text{firm nonexpansiveness of $\cJ$})
\end{align*} 
Hence, \cref{firm-nonexp} holds as desired. 

Now we prove \cref{seq-bd}. It suffices to consider the case where $k\ge1$. We have 
\begin{align}
\|(w^{k},\mu^{k}) - (w^*,\mu^*)\| \le &\ \|(w^{k},\mu^{k}) - \cJ(w^{k-1},\mu^{k-1})\| + \|\cJ(w^{k-1},\mu^{k-1}) - \cJ(w^*,\mu^*)\|\nonumber\\
\le&\ \beta\sqrt{n} \tau_{k-1} + \|(w^{k-1},\mu^{k-1})-(w^*,\mu^*)\|, \label{inter-upbd-dist}
\end{align}
where we have invoked \cref{lem:tech1} and the nonexpansiveness of $\cJ$ to obtain the final upper bound. Repeatedly applying \cref{inter-upbd-dist} for iterates $(w^{1},\mu^{1})$ through $(w^{k-1},\mu^{k-1})$, we have 
\begin{align}
    \beta\sqrt{n} \tau_{k-1} + \|(w^{k-1},\mu^{k-1})-(w^*,\mu^*)\| \le \beta\sqrt{n}\sum_{j=0}^{k-1}\tau_j + \|(w^0,\mu^0)-(w^*,\mu^*)\|, 
\end{align}
completing the proof. 
\end{proof}

\begin{proof}[Proof of \cref{lem:outloop-bd}]
Notice from \cref{alg:g-admm-cvx} that $\tau_k=\bar{s}/(k+1)^2$ for all $k\ge 0$. Therefore, one has $\sum_{j=0}^\infty \tau_j\le 2\bar{s}$. In view of this, \cref{r0-theta}, and \cref{lem:tech2}, we observe that
\begin{equation}\label{upbd:wk-wstar}
\max\{\|w^k-w^*\|,\|\mu^k-\mu^*\|,\|w^k-w^k_*\|,\|w^k_*-w^*\|\}\le r_0+2\sqrt{n}\bar{s}\beta,\quad \forall k\ge0.   
\end{equation}
where $r_0$ is defined in \cref{r0-theta}, and $\beta$ and $\bar{s}$ are inputs of \cref{alg:g-admm-cvx}. \cref{upbd:wk-wstar} implies that $w^k\in\cQ_1$ for all $k\ge0$, completing the proof.   
\end{proof}

\subsection{Proof of \cref{thm:outer-complex}(a)}\label{apx:cplx-pal}

To prove \cref{thm:outer-complex}(a), we first present a general technical lemma on the convergence rate of an inexact PPA applied to monotone inclusion problems. 

\begin{lemma}[{{\bf restatement of Lemma~3 of \cite{LZ18AL}}}]\label{lem:inclusion}
Let $\widetilde{\cT}:\bR^{p}\rightrightarrows\bR^{q}$ be a maximally monotone operator and $z^*\in\bR^p$ such that $0\in\widetilde{\cT}(z^*)$. Let $\{z^k\}$ be a sequence generated by an inexact PPA, starting with $z^0$ and obtaining $z^{k+1}$ by approximately evaluating $\widetilde{\cJ}(z^k)$ such that 
\begin{equation*}
\|z^{k+1}-\widetilde{\cJ}(z^k)\|\le e_k    
\end{equation*}
for some $\beta>0$ and $e_k\ge 0$, where $\widetilde{\cJ}:=(\cI+\beta \widetilde{\cT})^{-1}$ and $\cI$ is the identity operator. Then, for any $K\ge 1$, we have
\begin{equation*}
\min_{K\le k\le 2K}\|z^{k+1}-z^k\| \le \frac{\sqrt{2}\left(\|z^0-z^*\| + 2\sum_{k=0}^{2K}e_k\right)}{\sqrt{K+1}}.
\end{equation*}
\end{lemma}

\begin{proof}[Proof of \cref{thm:outer-complex}(a)]
Observe that \cref{alg:g-admm-cvx} terminates when two consecutive iterates $(w^{k+1},\mu^{k+1})$ and $(w^k,\mu^k)$ are close. We use this observation and \cref{lem:tech1,lem:inclusion} to derive the maximum number of outer iterations of \cref{alg:g-admm-cvx}.

Recall that $\sum_{j=0}^\infty\tau_j \le 2\bar{s}$. It follows from \cref{lem:tech1,lem:inclusion} that 
\begin{align*}
\min_{K\le k\le 2K}\frac{1}{\beta}\|(w^{k+1},\mu^{k+1})-(w^k,\mu^k)\|&\le \frac{\sqrt{2}\left(\|(w^0,\mu^0)-(w^*,\mu^*)\|+2\sqrt{n}\beta\sum_{j=0}^\infty\tau_j\right)}{\beta\sqrt{K+1}}\\
&\le \frac{\sqrt{2}\left(\|(w^0,\mu^0)-(w^*,\mu^*)\|+4\sqrt{n}\bar{s}\beta \right)}{\beta\sqrt{K+1}}=\frac{\sqrt{2}\left(r_0+4\sqrt{n}\bar{s}\beta \right)}{\beta\sqrt{K+1}},
\end{align*}
which then implies that
\begin{align*}
&\min_{K\le k\le 2K}\left\{\tau_k+\frac{1}{\beta}\|w^{k+1}-w^k\|_{\infty}\right\}\le \frac{\bar{s}}{(K+1)^2} +  \frac{\sqrt{2}\left(r_0+4\sqrt{n}\bar{s}\beta \right)}{\beta\sqrt{K+1}}\le \left[\bar{s}+\frac{\sqrt{2}\left(r_0+4\sqrt{n}\bar{s}\beta \right)}{\beta}\right]\frac{1}{\sqrt{K+1}},\\ 
&\min_{K\le k\le 2K}\frac{1}{\beta}\|\mu^{k+1}-\mu^k\|_{\infty}\le \frac{\sqrt{2}\left(r_0+4\sqrt{n}\bar{s}\beta \right)}{\beta\sqrt{K+1}}.
\end{align*}
We see from these and the termination criterion in \cref{stop-alg} that the number of outer iterations of \cref{alg:g-admm-cvx} is at most
\begin{equation}\label{def:specific-total-outer}
K_{\eps_1,\eps_2}:=\left[\bar{s}+\frac{\sqrt{2}(r_0+4\sqrt{n}\bar{s}\beta)}{\beta}\right]^2\max\{\epsilon_1^{-2},\epsilon_2^{-2}\}=\cO(\max\{\epsilon_1^{-2},\epsilon_2^{-2}\}).
\end{equation}
Hence, \cref{thm:outer-complex}(a) holds as desired.
\end{proof}

\section{Proofs of the main results in \cref{sec:admm}}\label{apx:pf-admm}

Throughout this section, we let $(\tilde{w}^*,u^*)$ be the optimal solution of \cref{pb:consens}, and $\lambda^*$ be the associated Lagrangian multiplier. Recall from the definition of $\tilde{u}_i^0$ in \cref{alg:admm-cvx-1} and \cref{def:tuit} that
\begin{equation}\label{def:tuitall}
\tilde{u}^t_i=u_i^t + \lambda^t_i/\rho_i,\quad \forall 1\le i\le n, t\ge 0.
\end{equation}

\subsection{Proof of \cref{lem:local-Lip}}\label{apx:lcl-lip}
For notational convenience, write $f_0(w)\equiv0$. Then, by \cref{def:P0k,def:Pik}, one can verify that
\begin{equation}\label{grad-Pik}
\nabla P_{i,k}(w) = \nabla f_i(w) + \nabla c_i(w) [\mu_i^k + \beta c_i(w)]_+ + \frac{1}{(n+1)\beta}(w-w^k),\quad \forall 0\le i\le n.
\end{equation}
\begin{proof}[Proof of \cref{lem:local-Lip}]
Fix an arbitrary $w\in\bR^d$ and a bounded open set $\cU_w$ containing $w$. We suppose that $\nabla f_i$ is $L_{w,1}$-Lipschitz continuous on $\cU_w$, and $\nabla c_i$ is $L_{w,2}$-Lipschitz continuous on $\cU_w$. Also, let $U_{w,1}=\sup_{w\in\cU_w}\|c_i(w)\|$ and $U_{w,2}=\sup_{w\in\cU_w}\|\nabla c_i(w)\|$. By \cref{def:P0k,def:Pik,grad-Pik}) one has for each $0\le i\le n$ and $u,v\in\cU_w$ that
\begin{align}
\|\nabla P_{i,k}(u)-\nabla P_{i,k}(v)\| \overset{\cref{grad-Pik}}{\le}&\ \|\nabla f_i(u)-\nabla f_i(v)\| + \|\nabla c_i(u)-\nabla c_i(v)\|\|[\mu_i^k + \beta c_i(u)]_+\|\nonumber\\
&\ +\|[\mu_i^k + \beta c_i(u)]_+-[\mu_i^k + \beta c_i(v)]_+\|\|\nabla c_i(v)\| + \frac{1}{(n+1)\beta}\|u-v\|\nonumber\\
\le&\ L_{w,1}\|u-v\|+(\|\mu_i^k\|+\beta U_{w,1})L_{w,2}\|u-v\|\nonumber \\
&\ + \beta\|c_i(u)-c_i(v)\|\|\nabla c_i(v)\| + \frac{1}{(n+1)\beta}\|u-v\|\nonumber\\
\le&\ \left[L_{w,1}+(\|\mu_i^k\|+\beta U_{w,1})L_{w,2}+\beta U_{w,2}^2+\frac{1}{(n+1)\beta}\right]\|u-v\|.\nonumber
\end{align}
Therefore, $\nabla P_{i,k}(u)$ is locally Lipschitz continuous on $\bR^d$, and the conclusion holds as desired.
\end{proof}

\subsection{Proof of \cref{thm:output-alg1}}\label{apx:out-admm}
\begin{proof}[Proof of \cref{thm:output-alg1}]
In view of the termination criterion~\cref{sbpb:stop}, it suffices to show that \[
\rmdist_\infty(0,\partial \ell(w^{T+1}))\le \varepsilon_{T+1} + \sum_{i=1}^n \tilde{\varepsilon}_{i,T+1}.
\]
By the definition of $\ell$ in \cref{pb:fs-gl}, one has that	
\begin{equation}\label{partial-Fh}
\partial \ell(w^{T+1}) = \sum_{i=0}^n\nabla P_i(w^{T+1}) + \partial h(w^{T+1}).
\end{equation}
In addition, notice from \cref{sbpb:phi-gt,sbpb:phi-it,def:tuitall} that
\begin{align*}
\partial\varphi_{0,T}(w^{T+1}) = &\ \nabla P_0(w^{T+1}) + \sum_{i=1}^n\rho_i(w^{T+1}-\tilde{u}_i^T) + \partial h(w^{t+1})\\
=&\ \nabla P_0(w^{T+1}) + \sum_{i=1}^n[\rho_i(w^{T+1}-u_i^{T})-\lambda_i^T] + \partial h(w^{T+1}),\\
\nabla\varphi_{i,T}(w^{T+1}) = &\ \nabla P_i(w^{T+1}) + \lambda_i^T,\quad \forall 1\le i\le n.
\end{align*}
Combining these with \cref{partial-Fh}, we obtain that
\[
\partial \ell(w^{T+1}) = \partial\varphi_{0,T}(w^{T+1}) + \sum_{i=1}^n [\nabla\varphi_{i,T}(w^{T+1}) - \rho_i(w^{T+1}-u_i^T)],
\]
which together with $\rmdist_\infty(0,\partial\varphi_{0,T}(w^{T+1}))\le\varepsilon_{T+1}$ (see \cref{alg:admm-cvx-1,vareps-gt}) implies that
\begin{align*}
\rmdist_\infty(0,\partial \ell(w^{T+1}))\le&\ \rmdist_\infty(0,\partial\varphi_{0,T}(w^{T+1})) + \sum_{i=1}^n \|\nabla\varphi_{i,T}(w^{T+1}) - \rho_i(w^{T+1}-u_i^T)\|_\infty\\
\le&\ \varepsilon_{T+1} + \sum_{i=1}^n \tilde{\varepsilon}_{i,T+1}, 
\end{align*}
as desired.
\end{proof}

\subsection{Proof of \cref{lem:uit-wt-bd}}\label{apx:bd-admm}

To prove \cref{lem:uit-wt-bd}, we use convergence analysis techniques for ADMM to show that the distances between iterates $\{u_i^k\}_{1\le i\le n}$ and $w^k$ and the optimal solution $\tilde{w}^*$ are controlled by the distance between the initial iterate $\tilde{w}^0$ and $\tilde{w}^*$. To the best of our knowledge, such boundedness results without assuming global Lipschitz continuity are entirely new in the literature on ADMM.

\begin{proof}[Proof of \cref{lem:uit-wt-bd}]

From the optimality conditions and stopping criteria for \cref{sbpb:phi-gt} and \cref{sbpb:phi-it}, there exist $e_i^{t+1}$'s for $0 \le i \le n$ with  $\|e_i^{t+1}\|_{\infty}\le\varepsilon_{t+1}$ and $h^{t+1}\in\partial h(w^{t+1})$ so that: 
\begin{align}
e^{t+1}_0 =&\ \nabla P_0(w^{t+1}) + h^{t+1} + \sum_{i=1}^n \rho_i(w^{t+1}-\tilde{u}_i^{t}) \overset{\cref{def:tuitall}}{=}  \nabla P_0(w^{t+1}) + h^{t+1} + \sum_{i=1}^n[\rho_i(w^{t+1}-u_i^{t})-\lambda_i^{t}]\nonumber\\
& \quad \overset{\cref{sbpb:lam-gt}}{=}\ \nabla P_0(w^{t+1}) + h^{t+1} + \sum_{i=1}^n[\rho_i(u_i^{t+1}-u_i^{t})-\lambda_i^{t+1}] \label{opt-Fg}
\end{align}
and 
\begin{equation}
e^{t+1}_i = \nabla\varphi_{i,t}(u_i^{t+1}) \overset{\cref{sbpb:phi-it}}{=} \nabla P_{i}(u_i^{t+1}) + \lambda_i^{t} + \rho_i(u_i^{t+1} - w^{t+1}) \overset{\cref{sbpb:lam-gt}}{=} \nabla P_i(u_i^{t+1}) + \lambda_i^{t+1} ,\quad \forall 1\le i\le n\label{opt-Fi}. 
\end{equation}
Moreover, since $\tilde{w}^*$ and $u^*$ are the optimal solution of \cref{pb:consens} with the associated Lagrangian multiplier $\lambda^*\in\bR^m$, we have by the optimality condition that there exists $h^*\in\partial h(\tilde{w}^*)$ such that
\begin{equation}\label{opt-starF}
\nabla P_i(u^*_i) + \lambda_i^*=0,\quad \nabla P_0(\tilde{w}^*) + h^* - \sum_{i=1}^n\lambda_{i}^*=0,\quad  u_i^*=\tilde{w}^*,\quad \forall 1\le i\le n.
\end{equation}
Recall that $P_i$, $0\le i\le n$, are strongly convex with the modulus $\sigma>0$, we have 
\begin{align*}
\sigma\|u_i^{t+1}- \tilde{w}^*\|^2
&\le  \langle u_i^{t+1} - \tilde{w}^*, \nabla P_i(u_i^{t+1})-\nabla P_i(\tilde{w}^*)\rangle \quad \quad (\text{strong convexity of $P_i$})\nonumber \\
& = \langle u_i^{t+1} - \tilde{w}^{*}, -\lambda_i^{t+1} + \lambda_i^* + e_i^{t+1}\rangle  \quad \quad  (\text{$\tilde{w}^*=u^*_i$, $\nabla P_i(u^*_i)=\lambda^*_i$, and \cref{opt-Fi}}) \nonumber\\
&\le \langle u_i^{t+1} - \tilde{w}^*, -\lambda_i^{t+1} + \lambda_{i}^*\rangle + \frac{\sigma}{2}\|u_i^{t+1} - \tilde{w}^*\|^2 + \frac{1}{2\sigma}\|e_i^{t+1}\|^2,  \\ 
& \quad \quad \quad \quad  (\text{$\langle a,b\rangle\le t/2\|a\|^2 + 1/(2t)\|b\|^2$ for all $a,b\in\bR^d$ and $t>0$}),     
\end{align*}
and 
\begin{align*}
\sigma\|w^{t+1} - \tilde{w}^*\|^2\le &\  \langle w^{t+1} - \tilde{w}^*, \nabla P_0(w^{t+1}) + h^{t+1}-\nabla P_0(\tilde{w}^*) - h^*\rangle  \quad (\text{strong convexity of $P_0+ h$})\\
=&\ \langle w^{t+1} - \tilde{w}^*, \sum_{i=1}^n[\lambda_i^{t+1}-\lambda_i^* -\rho_i(u_i^{t+1} - u_i^{t})] + e_0^{t+1}\rangle, \quad (\text{\cref{opt-Fg} and \cref{opt-starF}})\\
\le&\ \langle w^{t+1} - \tilde{w}^*, \sum_{i=1}^n[\lambda_i^{t+1}-\lambda^*_i -\rho_i(u_i^{t+1} - u_i^{t})]\rangle + \frac{\sigma}{2}\|w^{t+1} - \tilde{w}^*\|^2 + \frac{1}{2\sigma}\|e_0^{t+1}\|^2, \\
& \quad \quad \quad \quad  (\text{$\langle a,b\rangle\le t/2\|a\|^2 + 1/(2t)\|b\|^2$ for all $a,b\in\bR^d$ and $t>0$}).      
\end{align*}
Summing up these inequalities and rearranging the terms, we obtain that 
\begin{align}
&\frac{\sigma}{2}(\|w^{t+1} - \tilde{w}^*\|^2+\sum_{i=1}^n \|u_i^{t+1} - \tilde{w}^*\|^2)\nonumber\\
&\le \sum_{i=1}^n \langle w^{t+1} - \tilde{w}^*, \lambda_i^{t+1}-\lambda_{i}^*-\rho_i(u_i^{t+1} - u_i^{t})\rangle + \frac{1}{2\sigma}\|e_0^{t+1}\|^2 + \sum_{i=1}^n (\langle u_i^{t+1} - \tilde{w}^*, -\lambda_i^{t+1} + \lambda_{i}^*\rangle + \frac{1}{2\sigma}\|e_i^{t+1}\|^2 ) \nonumber\\
& \le \sum_{i=1}^n \langle w^{t+1} - u_i^{t+1}, \lambda_i^{t+1} - \lambda_{i}^* \rangle + \sum_{i=1}^n\rho_i\langle w^{t+1} - \tilde{w}^*, u_i^{t} - u_i^{t+1} \rangle + \frac{n+1}{2\sigma}\varepsilon_{t+1}^2 \nonumber\\
& \quad \quad \quad \quad  (\text{$\|e_i^{t+1}\|\le\varepsilon_{t+1}$ for all $0\le i\le n$ and $t\ge0$}) \nonumber \\
& \overset{\cref{sbpb:lam-gt}}{=} \sum_{i=1}^n \frac{1}{\rho_i} \langle \lambda_i^{t} - \lambda_i^{t+1}, \lambda_i^{t+1} - \lambda_i^* \rangle + \sum_{i=1}^n\rho_i \langle w^{t+1} - \tilde{w}^*, u_i^{t} - u_i^{t+1} \rangle+\frac{n+1}{2\sigma}\varepsilon_{t+1}^2,\label{upbd-uit}
\end{align}
Notice that the following well-known identities hold:
\begin{align}
&\langle w^{t+1} - \tilde{w}^*, u_i^t-u_i^{t+1}\rangle =  \frac{1}{2}(\|w^{t+1} - u_i^{t+1}\|^2-\|w^{t+1} - u_i^{t}\|^2 + \|\tilde{w}^* - u_i^{t}\|^2-\|\tilde{w}^* - u_i^{t+1}\|^2),\label{para-id}\\
&\langle\lambda_i^t-\lambda_i^{t+1}, \lambda_i^{t+1}-\lambda_i^*\rangle=\frac{1}{2}(\|\lambda_i^* - \lambda_i^{t}\|^2-\|\lambda_i^* - \lambda_i^{t+1}\|^2-\|\lambda_i^t-\lambda_i^{t+1}\|^2).\label{tri-id}
\end{align}
These along with \cref{sbpb:lam-gt,upbd-uit} imply that
\begin{align}
&\frac{\sigma}{2}(\|w^{t+1} - \tilde{w}^*\|^2+\sum_{i=1}^n \|u_i^{t+1} - \tilde{w}^*\|^2)+ \sum_{i=1}^n\frac{\rho_i}{2}\|w^{t+1}-u_i^t\|^2-\frac{n+1}{2\sigma}\varepsilon_{t+1}^2\nonumber\\
&\overset{\cref{upbd-uit})}{\le}\sum_{i=1}^n \frac{1}{\rho_i} \langle \lambda_i^{t} - \lambda_i^{t+1}, \lambda_i^{t+1} - \lambda_i^* \rangle + \sum_{i=1}^n\rho_i \langle w^{t+1} - \tilde{w}^*, u_i^{t} - u_i^{t+1} \rangle + \sum_{i=1}^n\frac{\rho_i}{2}\|w^{t+1}-u_i^t\|^2\nonumber\\
&\overset{\cref{para-id}}{\le}\sum_{i=1}^n\frac{1}{\rho_i}\langle\lambda_i^t-\lambda_i^{t+1}, \lambda_i^{t+1}-\lambda_i^*\rangle + \sum_{i=1}^n\frac{\rho_i}{2} (\|\tilde{w}^* - u_i^{t}\|^2-\|\tilde{w}^* - u_i^{t+1}\|^2+\|w^{t+1} - u_i^{t+1}\|^2)  \nonumber\\
&\overset{\cref{sbpb:lam-gt}}{=}\sum_{i=1}^n\frac{1}{\rho_i}\langle\lambda_i^t-\lambda_i^{t+1}, \lambda_i^{t+1}-\lambda_i^*\rangle + \sum_{i=1}^n\frac{1}{2\rho_i}\|\lambda_i^{t+1}-\lambda_i^t\|^2 +\sum_{i=1}^n\frac{\rho_i}{2} (\|\tilde{w}^* - u_i^{t}\|^2-\|\tilde{w}^* - u_i^{t+1}\|^2) \nonumber\\
&\overset{\cref{tri-id}}{=} \sum_{i=1}^n\frac{1}{2\rho_i}(\|\lambda_i^* - \lambda_i^{t}\|^2-\|\lambda_i^* - \lambda_i^{t+1}\|^2) +  \sum_{i=1}^n\frac{\rho_i}{2}(\|\tilde{w}^* - u_i^{t}\|^2-\|\tilde{w}^* - u_i^{t+1}\|^2) \nonumber\\
&= \sum_{i=1}^n[(\frac{\rho_i}{2}\|\tilde{w}^* - u_i^{t}\|^2 + \frac{1}{2\rho_i} \|\lambda_{i}^* - \lambda_i^{t}\|^2) - (\frac{\rho_i}{2}\|\tilde{w}^* - u_i^{t+1}\|^2+\frac{1}{2\rho_i}\|\lambda_i^* - \lambda_i^{t+1}\|^2)].\label{uit-wt-upbd}
\end{align}
Summing up this inequality over $t=0,\ldots,T$ for any $T \ge 0$, we obtain that
\begin{align}
&\sum_{t=0}^{T}\left[\frac{\sigma}{2}\left(\|w^{t+1}-\tilde{w}^*\|^2+\sum_{i=1}^n\|u_i^{t+1}-\tilde{w}^*\|^2\right) + \sum_{i=1}^n\frac{\rho_i}{2}\|w^{t+1}-u_i^t\|^2-\frac{n+1}{2\sigma}\varepsilon_{t+1}^2\right]\nonumber\\
&\le\sum_{i=1}^n\left[\left(\frac{\rho_i}{2}\|\tilde{w}^* - u_i^{0}\|^2 + \frac{1}{2\rho_i} \|\lambda^*_i - \lambda_i^0\|^2\right) - \left(\frac{\rho_i}{2}\|\tilde{w}^* - u_i^{T+1}\|^2+\frac{1}{2\rho_i}\|\lambda_i^* - \lambda_i^{T+1}\|^2\right)\right].\label{sum-desc}
\end{align}
Recall from \cref{alg:admm-cvx-1} that $\varepsilon_{t+1}=q^t$, $u_i^0=\tilde{w}^0$, and $\lambda_i^0=-\nabla P_i(\tilde{w}^0)$. Notice from \cref{opt-starF} that $\tilde{w}^*=u_i^*$ and $\lambda_i^*=-\nabla P_i(u_i^*)$. By these and \cref{sum-desc}, one can deduce that
\begin{align}
&\frac{\sigma}{2}(\|w^{t+1}-\tilde{w}^*\|^2+\sum_{i=1}^n\|u_i^{t+1} - \tilde{w}^*\|^2)\le \frac{n+1}{2\sigma}\sum_{t=0}^\infty q^{2t} + \sum_{i=1}^n\left(\frac{\rho_i}{2}\|\tilde{w}^* - u_i^0\|^2 + \frac{1}{2\rho_i} \|\lambda_i^* - \lambda_i^0\|^2\right)\nonumber\\
&\le \frac{n+1}{2\sigma(1-q^2)} + \sum_{i=1}^n \left(\frac{\rho_i}{2}\|\tilde{w}^* - u_i^0\|^2 + \frac{1}{2\rho_i} \|\lambda_i^* - \lambda_i^{0}\|^2\right)\nonumber\\
&= \frac{n+1}{2\sigma(1-q^2)} + \sum_{i=1}^n\left(\frac{\rho_i}{2}\|\tilde{w}^* - \tilde{w}^0\|^2 + \frac{1}{2\rho_i} \|\nabla P_i(\tilde{w}^*) - \nabla P_i(\tilde{w}^0)\|^2\right).\nonumber
\end{align}
In view of this and the definition of $\cQ$ in \cref{def:Q_Fh}, we can observe that $w^{t+1}\in\cQ$ and $u_i^{t+1}\in\cQ$ for all $t\ge0$ and $1\le i\le n$. Hence, the conclusion of this lemma holds as desired.
\end{proof}

\subsection{Proof of \cref{thm:complexity-alg1}}\label{apx:cplx-admm}
We first prove an auxiliary recurrence result that will be used later.
\begin{lemma}\label{lem:induct}
Assume that $r,c>0$ and $q\in(0,1)$. Let $\{a_t\}_{t\ge 0}$ be a sequence satisfying
\begin{equation}\label{seq:at}
(1+r) a_{t+1}\le a_t + c q^{2t},\quad \forall t\ge 0.
\end{equation}
Then we have
\begin{equation}\label{linear-at}
a_{t+1}\le \max\left\{q,\frac{1}{1+r}\right\}^{t+1} \left(a_0 + \frac{c}{1-q}\right),\quad \forall t\ge 0.
\end{equation}
\end{lemma}

\begin{proof}
It follows \cref{seq:at} that
\begin{align*}
a_{t+1}&\le \frac{1}{1+r}a_t + \frac{1}{1+r} cq^{2t}\le \frac{1}{(1+r)^2}a_{t-1} + \frac{cq^{2(t-1)}}{(1+r)^2} + \frac{cq^{2t}}{1+r} \\
&\le\cdots\le \frac{1}{(1+r)^{t+1}}a_0 + \sum_{i=0}^t \frac{c q^{2i}}{(1+r)^{t+1-i}} = \frac{1}{(1+r)^{t+1}}a_0 + c\sum_{i=0}^t \frac{q^{i}}{(1+r)^{t+1-i}} q^{i}\\
&\le \frac{1}{(1+r)^{t+1}}a_0 + c\max\left\{q,\frac{1}{1+r}\right\}^{t+1}\sum_{i=0}^tq^{i}\\
&\quad\quad\quad\quad\quad (q^i\le \max\{q,1/(1+r)\}^i\text{ and } 1/(1+r)^{t+1-i}\le \max\{q,1/(1+r)\}^{t+1-i})\\
&\le \frac{1}{(1+r)^{t+1}}a_0 + \frac{c}{1-q}\max\left\{q,\frac{1}{1+r}\right\}^{t+1}\\
&\le \max\left\{q,\frac{1}{1+r}\right\}^{t+1} \left(a_0 + \frac{c}{1-q}\right).
\end{align*}
Hence, \cref{linear-at} holds as desired.
\end{proof}

The following lemma proves the Lipschitz continuity of $\nabla P_i$ on $\cQ$.

\begin{lemma}\label{lem:Lip-gradF}
Let $\cQ$ be defined in \cref{def:Q_Fh}. Then there exists some $L_{\nabla P}>0$ such that
\begin{equation}\label{Lip-gradF}
\|\nabla P_i(u)-\nabla P_i(v)\| \le L_{\nabla P}\|u-v\|,\quad \forall u,v\in\cQ, 0\le i\le n.
\end{equation}
\end{lemma}

\begin{proof}
Notice from \cref{def:Q_Fh} that the set $\cQ$ is convex and compact. By this and the local Lipschitz continuity of $\nabla P_i$ on $\bR^d$, one can verify that there exists some constant $L_{\nabla P}>0$ such that \cref{Lip-gradF} holds (see also Lemma~1 in \cite{LM22}).
\end{proof}

We introduce a potential function $S_t$ to measure the convergence of \cref{alg:admm-cvx-1}:
\begin{equation}
S_t := \sum_{i=1}^n\left(\frac{\rho_i}{2}\|\tilde{w}^*-u_i^t\|^2+\frac{1}{2\rho_i}\|\lambda_i^*-\lambda_i^t\|^2\right),\quad \forall t\ge 0.\label{def:St}    
\end{equation}
The following lemma gives a recursive result of $S_t$, which will play a key role on establishing the global convergence rate for \cref{alg:admm-cvx-1} in \cref{thm:complexity-alg1}.

\begin{lemma}\label{lem:St-rec}
Suppose that Assumptions~\ref{asp:cvx-lclc} to \ref{asp:out-analysis} hold. Let $\{w^{t+1}\}_{t\ge0}$ and $\{u_i^{t+1}\}_{1\le i\le n,t\ge0}$ be all the iterates generated by \cref{alg:admm-cvx-1}. Then we have 
\begin{equation}\label{F-geo-converge}
S_t \le q_r^t\left[S_0+\frac{1}{1-q}\left(\frac{n+1}{2\sigma}+\sum_{i=1}^n\frac{\sigma}{\rho_i^2+2L_{\nabla P}^2}\right)\right],\quad \forall t\ge 0,
\end{equation}
where $\sigma$ and $L_{\nabla P}$ are given in \cref{Pi-s-convex} and \cref{lem:Lip-gradF}, respectively, $q$ and $\rho_i$, $1\le i\le n$, are inputs of \cref{alg:admm-cvx-1}, and 
\begin{equation}
q_r:=\max\left\{q,\frac{1}{1+r}\right\},\quad r := \min_{1\le i\le n}\left\{\frac{\sigma\rho_i}{\rho_i^2+2L_{\nabla P}^2}\right\}\label{def:rrq}.
\end{equation}
\end{lemma}

\begin{proof}
Recall from \cref{uit-wt-upbd} that 
\begin{align}
S_t = &\sum_{i=1}^n\left(\frac{\rho_i}{2}\|\tilde{w}^* - u_i^t\|^2 + \frac{1}{2\rho_i} \|\lambda_i^* - \lambda_i^t\|^2\right)\nonumber\\
\ge & \sum_{i=1}^n\left(\frac{\rho_i+\sigma}{2}\|\tilde{w}^* - u_i^{t+1}\|^2+\frac{1}{2\rho_i}\|\lambda_i^* - \lambda_i^{t+1}\|^2+\frac{\rho_i}{2}\|w^{t+1}-u_i^t\|^2\right)+\frac{\sigma}{2}\|w^{t+1} - \tilde{w}^*\|^2-\frac{n+1}{2\sigma}\varepsilon_{t+1}^2\nonumber\\
\ge & \sum_{i=1}^n\left(\frac{\rho_i+\sigma}{2}\|\tilde{w}^* - u_i^{t+1}\|^2+\frac{1}{2\rho_i}\|\lambda_i^* - \lambda_i^{t+1}\|^2\right)-\frac{n+1}{2\sigma}\varepsilon_{t+1}^2 .\label{rec-wlam-inner}
\end{align}
Also, notice from \cref{opt-Fi,opt-starF,Lip-gradF} that
\[
\|\lambda_i^* - \lambda_i^{t+1}\|^2\overset{\cref{opt-Fi,opt-starF}}{\le}(\|\nabla P_i(\tilde{w}^*)-\nabla P_i(u_i^{t+1})\| + \|e_i^{t+1}\|)^2\overset{\cref{Lip-gradF}}{\le} 2L_{\nabla P}^2\|\tilde{w}^{*}-u_i^{t+1}\|^2 + 2\varepsilon_{t+1}^2,
\]
which implies that
\begin{equation}\label{upbdwu-10086}
\|\tilde{w}^*-u_i^{t+1}\|^2\ge \frac{2\rho_i}{\rho_i^2+2L_{\nabla P}^2}\left(\frac{\rho_i}{2}\|\tilde{w}^*-u_i^{t+1}\|^2+\frac{1}{2\rho_i}\|\lambda_i^*-\lambda_i^{t+1}\|^2\right) - \frac{2\varepsilon^2_{t+1}}{\rho_i^2+2L_{\nabla P}^2}.
\end{equation}
Plugging this into \cref{rec-wlam-inner}, we have 
\begin{align*}
S_t 
&\overset{\cref{upbdwu-10086}}{\ge}\sum_{i=1}^n\left(1+\frac{\sigma\rho_i}{\rho_i^2+2L_{\nabla P}^2}\right)\left(\frac{\rho_i}{2}\|\tilde{w}^*-u_i^{t+1}\|^2+\frac{1}{2\rho_i}\|\lambda_i^*-\lambda_i^{t+1}\|^2\right) - \frac{n+1}{2\sigma}q^{2t} - \sum_{i=1}^n\frac{\sigma}{\rho_i^2+2L_{\nabla P}^2} q^{2t}  \\
& =  (1+r)S_{t+1} - \left(\frac{n+1}{2\sigma}+\sum_{i=1}^n\frac{\sigma}{\rho_i^2+2L_{\nabla P}^2}\right)q^{2t}\quad\quad\quad\quad (r := \min_{1 \le i \le n} \sigma\rho_i/(\rho_i^2+2L_{\nabla P}^2)).
\end{align*} 
When $t=0$, \cref{F-geo-converge} holds clearly. When $t\ge 1$, by the above inequality, \cref{def:rrq}, and \cref{lem:induct} with $(a_t,c)=(S_t,\frac{n+1}{2\sigma}+\sum_{i=1}^n\frac{\sigma}{\rho_i^2+2L_{\nabla P}^2})$, we obtain that
\begin{align*}
S_t\le&\ \max\left\{q,\frac{1}{1+r}\right\}^t\left[S_0+\frac{1}{1-q}\left(\frac{n+1}{2\sigma}+\sum_{i=1}^n\frac{\sigma}{\rho_i^2+2L_{\nabla P}^2}\right)\right]\\
=&\ q_r^t\left[S_0+\frac{1}{1-q}\left(\frac{n+1}{2\sigma}+\sum_{i=1}^n\frac{\sigma}{\rho_i^2+2L_{\nabla P}^2}\right)\right].
\end{align*}
Therefore, the conclusion of this lemma is true as desired.
\end{proof}

\begin{proof}[Proof of \cref{thm:complexity-alg1}]
Notice that \cref{alg:admm-cvx-1} terminates when $\varepsilon_{t+1}+\sum_{i=1}^n\tilde{\varepsilon}_{i,t+1}$ is small. Next, we show that this quantity is bounded by $S_t$ defined in \cref{def:St} plus other small quantities, and then use \cref{lem:St-rec} to bound the maximum number of iterations of \cref{alg:admm-cvx-1}.

By \cref{vareps-gt}, and the fact that $\|\nabla\varphi_{i,t}(u_i^{t+1})\|_{\infty}\le\varepsilon_{t+1}$ (see \cref{alg:admm-cvx-1}), one can obtain that
\begin{align}
\varepsilon_{t+1}+\sum_{i=1}^n\tilde{\varepsilon}_{i,t+1} & \overset{\cref{vareps-gt}}{=} \varepsilon_{t+1}+ \sum_{i=1}^n\|[\nabla\varphi_{i,t}(w^{t+1})-\rho_i(w^{t+1}-u_i^t)]\|_{\infty}\nonumber\\ 
&\le\varepsilon_{t+1}+\sum_{i=1}^n\|\nabla\varphi_{i,t}(u_i^{t+1})\|_{\infty}+\sum_{i=1}^n\|\nabla\varphi_{i,t}(w^{t+1})-\nabla\varphi_{i,t}(u_i^{t+1})\|+ \sum_{i=1}^n\rho_i\|w^{t+1}-u_i^t\|\nonumber\\
&\quad\quad\quad\quad\quad\quad (\text{$\|u\|_\infty\le\|u\|$ for all $u\in\bR^d$ and the triangle inequality})\nonumber\\
&\le (n+1)\varepsilon_{t+1} + \sum_{i=1}^n (L_{\nabla P}+\rho_i)\|w^{t+1}-u_i^{t+1}\| + \sum_{i=1}^n\rho_i\|w^{t+1}-u_i^t\|,\label{upbd-ter-alg1}
\end{align}
where the second inequality follows from
\begin{align*}
\|\nabla\varphi_{i,t}(w^{t+1})-\nabla\varphi_{i,t}(u_i^{t+1})\|\overset{\cref{sbpb:phi-it}}{\le}&\|\nabla P_i(w^{t+1})-\nabla P_i(u_i^{t+1})\| + \rho_i\|w^{t+1}-u_i^{t+1}\|\\ 
\overset{\cref{Lip-gradF}}{\le}&\left(L_{\nabla P}+\rho_i\right)\|w^{t+1}-u_i^{t+1}\|,\qquad \forall 1\le i\le n.
\end{align*}

Next, we derive upper bounds for $\|w^{t+1}-u_i^{t+1}\|$ and $\rho_i\|w^{t+1}-u_i^t\|$, respectively. First, by \cref{uit-wt-upbd,F-geo-converge,def:St}, we have 
\begin{align}
&\frac{\sigma}{4}\|w^{t+1}-u_i^{t+1}\|^2\le \frac{\sigma}{2}\|w^{t+1}-\tilde{w}^*\|^2	+ \frac{\sigma}{2}\|u_i^{t+1}-\tilde{w}^*\|^2\nonumber\\
&\overset{\cref{uit-wt-upbd}}{\le} \sum_{i=1}^n(\frac{\rho_i}{2}\|\tilde{w}^*-u_i^t\|^2+\frac{1}{2\rho_i}\|\lambda_i^*-\lambda_i^t\|^2) + \frac{n+1}{2\sigma}\varepsilon_{t+1}^2\\
&= S_t + \frac{n+1}{2\sigma}q^{2t}\qquad\qquad\qquad\text{(the definition of $S_t$ in \cref{def:St} and $\varepsilon_{t+1}=q^t$ for all $t\ge0$})\nonumber\\
&\overset{\cref{F-geo-converge}}{\le}  q_r^t\left[S_0+\frac{1}{1-q}\left(\frac{n+1}{2\sigma}+\sum_{i=1}^n\frac{\sigma}{\rho_i^2+2L_{\nabla P}^2}\right)\right] + \frac{n+1}{2\sigma}q^{2t}\nonumber\\
&\le \left\{q_r^{t/2}\left[S_0+\frac{1}{1-q}\left(\frac{n+1}{2\sigma}+\sum_{i=1}^n\frac{\sigma}{\rho_i^2+2L_{\nabla P}^2}\right)\right]^{1/2} + \sqrt{\frac{n+1}{2\sigma}}q^{t} \right\}^2\nonumber\\
&\qquad\qquad\qquad\qquad\qquad\qquad\qquad\qquad\qquad\qquad\qquad\qquad  (\text{$a^2+b^2\le (a+b)^2$ for all $a,b\ge0$}).\label{upbd-diff-wu}
\end{align}
Using again \cref{uit-wt-upbd,F-geo-converge,def:St}, we obtain that
\begin{align}
&\frac{1}{2}(\sum_{i=1}^n\rho_i\|w^{t+1}-u^t_i\|)^2\le(\sum_{i=1}^n\rho_i)(\sum_{i=1}^n\frac{\rho_i}{2}\|w^{t+1}-u^t_i\|^2) \quad\quad\quad\quad (\text{Cauchy-Schwarz inequality})\nonumber\\ 
&\overset{\cref{uit-wt-upbd}}{\le}(\sum_{i=1}^n\rho_i) \sum_{i=1}^n(\frac{\rho_i}{2}\|\tilde{w}^*-u_i^t\|^2+\frac{1}{2\rho_i}\|\lambda_i^*-\lambda_i^t\|^2) + (\sum_{i=1}^n\rho_i)\frac{n+1}{2\sigma}\varepsilon_{t+1}^2\nonumber\\
&=(\sum_{i=1}^n\rho_i)\left(S_t+ \frac{n+1}{2\sigma}q^{2t}\right)\qquad\quad\text{(the definition of $S_t$ in \cref{def:St} and $\varepsilon_{t+1}=q^t$ for all $t\ge0$})\nonumber\\
&\overset{\cref{F-geo-converge}}{\le}(\sum_{i=1}^n\rho_i) \left\{\left[S_0+\frac{1}{1-q}\left(\frac{n+1}{2\sigma}+\sum_{i=1}^n\frac{\sigma}{\rho_i^2+2L_{\nabla P}^2}\right)\right]q_r^t+ \frac{n+1}{2\sigma}q^{2t}\right\}\nonumber\\
&\le (\sum_{i=1}^n\rho_i)\left\{q_r^{t/2}\left[S_0+\frac{1}{1-q}\left(\frac{n+1}{2\sigma}+\sum_{i=1}^n\frac{\sigma}{\rho_i^2+2L_{\nabla P}^2}\right)\right]^{1/2} + \sqrt{\frac{n+1}{2\sigma}}q^{t} \right\}^2\nonumber\\
&\qquad\qquad\qquad\qquad\qquad\qquad\qquad\qquad\qquad\qquad\qquad\qquad  (\text{$a^2+b^2\le (a+b)^2$ for all $a,b\ge0$}).\label{upbd-diff-wut-1}
\end{align}
Combining \cref{upbd-ter-alg1,upbd-diff-wu,upbd-diff-wut-1}, we obtain that 
\begin{align}
\varepsilon_{t+1}+ \sum_{i=1}^n\tilde{\varepsilon}_{i,t+1}\le&\ (n+1)q^t + \left(\frac{2}{\sqrt{\sigma}}\sum_{i=1}^n(L_{\nabla P} +\rho_i)+\sqrt{2\sum_{i=1}^n\rho_i}\;\right)\nonumber\\
&\cdot \left\{\left[S_0+\frac{1}{1-q}\left(\frac{n+1}{2\sigma}+\sum_{i=1}^n\frac{\sigma}{\rho_i^2+2L_{\nabla P}^2}\right)\right]^{1/2}q_r^{t/2}+ \sqrt{\frac{n+1}{2\sigma}}q^t\right\}\nonumber\\
\le&\ \ (n+1)q_r^{t/2} + \left(\frac{2}{\sqrt{\sigma}}\sum_{i=1}^n(L_{\nabla P} +\rho_i)+\sqrt{2\sum_{i=1}^n\rho_i}\;\right)\nonumber\\ 
&\cdot \Bigg\{\left[S_0+\frac{1}{1-q}\left(\frac{n+1}{2\sigma}+\sum_{i=1}^n\frac{\sigma}{\rho_i^2+2L_{\nabla P}^2}\right)\right]^{1/2} + \sqrt{\frac{n+1}{2\sigma}}\Bigg\}q_r^{t/2}\nonumber\\
&\qquad\qquad\qquad\qquad\qquad\qquad\qquad\qquad\qquad\qquad\qquad\qquad \text{($q\le q_r\le q_r^{1/2}<1$).}\label{upbd:epst-local}
\end{align}
Recall from \cref{alg:admm-cvx-1,opt-starF} that $(u_i^0,\lambda_i^0)=(\tilde{w}^0,-\nabla P_i(\tilde{w}^0))$ and $\lambda_i^*=-\nabla P_i(\tilde{w}^*)$. By these and \cref{def:St}, one has 
\begin{equation}\label{def:S0}
S_0 = \sum_{i=1}^n\left(\frac{\rho_i}{2}\|\tilde{w}^*- \tilde{w}^0\|^2 + \frac{1}{2\rho_i}\|\nabla P_i(\tilde{w}^*)-\nabla P_i(\tilde{w}^0)\|^2\right).
\end{equation}
For convenience, denote
\begin{align*}
&b := \left(\frac{2}{\sqrt{\sigma}}\sum_{i=1}^n(L_{\nabla P} +\rho_i)+\sqrt{2\sum_{i=1}^n\rho_i}\;\right)&\nonumber\\
&\qquad\qquad\cdot\Bigg\{\Bigg[\sum_{i=1}^n\left(\frac{\rho_i}{2}\|\tilde{w}^*- \tilde{w}^0\|^2 + \frac{1}{2\rho_i}\|\nabla P_i(\tilde{w}^*)-\nabla P_i(\tilde{w}^0)\|^2\right)+\frac{1}{1-q}\left(\frac{n+1}{2\sigma}+\sum_{i=1}^n\frac{\sigma}{\rho_i^2+2L_{\nabla P}^2}\right)\Bigg]^{1/2}\nonumber\\
&\qquad\qquad\qquad\qquad\qquad\qquad\qquad\qquad\qquad\qquad\qquad\qquad\qquad\qquad\qquad\qquad\qquad\qquad\qquad\qquad + \sqrt{\frac{n+1}{2\sigma}}\Bigg\}.
\end{align*}
Using this, \cref{upbd:epst-local,def:S0}, we obtain that
\[
\varepsilon_{t+1}+\sum_{i=1}^n\tilde{\varepsilon}_{i,t+1}\le (n+1+b)q_r^{t/2}.
\]
This along with the termination criterion in \cref{sbpb:stop} implies that the number of iterations of \cref{alg:admm-cvx-1} is bounded above by 
\begin{equation}\label{iter-alg1}
\left\lceil\frac{2\log(\tau/(n+1 + b))}{\log q_r}\right\rceil + 1 = \cO(|\log\tau|).
\end{equation}
Hence, the conclusion of this theorem holds as desired.
\end{proof}

We observe from the proof of \cref{thm:complexity-alg1} that under Assumptions~\ref{asp:cvx-lclc} to \ref{asp:out-analysis}, the number of iterations of \cref{alg:admm-cvx-1} is bounded by the quantity in \cref{iter-alg1}.

\section{Proof of \cref{thm:outer-complex}(b)}\label{apx:pf-pal}
To establish the total inner-iteration complexity, we first show that all the inner iterates produced by \cref{alg:admm-cvx-1} for solving all the subproblems of form \cref{prox-AL-pb-compact} within \cref{alg:g-admm-cvx} are in a compact set (i.e., $\cQ_2$ later), and then estimate the Lipschitz modulus of $\nabla P_{i,k}$ for all $0\le i\le n$ and all $k$ over $\cQ_2$. Then we can bound the inner-iteration complexity based on the size of $\cQ_2$ and the Lipschitz modulus. 

\paragraph{Boundedness of all inner iterates}
 Recall from \cref{r0-theta} that $\cQ_1$ is a compact set. Let 
\begin{equation}\label{upbd:Ugd}
U_{\nabla f}:=\sup_{w\in\cQ_1}\max_{1\le i\le n}\|\nabla f_i(w)\|,\quad U_{\nabla c}:=\sup_{w\in\cQ_1}\max_{0\le i\le n}\|\nabla c_i(w)\|,\quad U_c:=\sup_{w\in\cQ_1}\max_{0\le i\le n}\|c_i(w)\|.
\end{equation}
Since $\nabla f_i$ for all $1\le i\le n$ are locally Lipschitz, they are Lipschitz on $\cQ_1$ with some $L_{\nabla f}>0$. Similarly, $\nabla c_i$ for all $0\le i\le n$ are Lispchitz on $\cQ_1$ with some modulus $L_{\nabla c}>0$. Hence, $\nabla P_{i,k}$ for all $0\le i\le n$ and all $k$ are Lipschitz continuous on $\cQ_1$: 
\begin{align} \label{def:tLip-gradP-Q}
\|\nabla P_{i,k}(u)-\nabla P_{i,k}(v)\| \overset{\cref{grad-Pik}}{\le}&\ \|\nabla f_i(u)-\nabla f_i(v)\| + \|[\mu_i^k+\beta c_i(u)]_+\|\|\nabla c_i(u)-\nabla c_i(v)\|\nonumber\\
&\ +\|[\mu_i^k+\beta c_i(u)]_+-[\mu_i^k+\beta c_i(v)]_+\|\|\nabla c_i(v)\| + \frac{1}{(n+1)\beta}\|u-v\|\nonumber\\
\overset{\cref{def:tupbd-c-gc}}{\le}&\ L_{\nabla f}\|u-v\| + (\|\mu^*_i\| + \|\mu^k_i-\mu^*_i\| + \beta U_c) L_{\nabla c}\|u-v\|\nonumber\\
&\ + \beta U_{\nabla c}^2\|u-v\| + \frac{1}{(n+1)\beta}\|u-v\|\nonumber\\
\overset{\cref{upbd:wk-wstar}}{\le} &\ L_{\nabla P,1} \|u-v\| 
\end{align}
where
\begin{equation}\label{def:tLip-gradP-Q1}
L_{\nabla P,1} := L_{\nabla f}+(\|\mu^*\| + r_0+2\sqrt{n}\bar{s}\beta + \beta U_c) L_{\nabla c}+\beta U_{\nabla c}^2+\frac{1}{(n+1)\beta}.
\end{equation}
The next lemma says that all the inner iterates generated by \cref{alg:admm-cvx-1} stay in a compact set.
\begin{lemma}\label{lem:bd-iters-sbsolver} 
Suppose that Assumptions~\ref{asp:cvx-lclc} to \ref{asp:out-analysis} hold and let $\{w^{k,t+1}\}_{t\ge0}$ and $\{u_i^{k,t+1}\}_{1\le i\le n,t\ge0}$ be all the inner iterates generated by \cref{alg:admm-cvx-1} for solving the subproblems of form \cref{prox-AL-pb-compact} in \cref{alg:g-admm-cvx}. Then it holds that all these iterates stay in a compact set $\cQ_2$, where
\begin{align}
\cQ_2 := \left\{v:\|v-u\|^2\le\frac{(n+1)^3\beta^2}{(1-q^2)}+ (n+1)\beta \sum_{i=1}^n\left[\left(\rho_i+\frac{L_{\nabla P,1}^2}{\rho_i}\right)(r_0 + 2\sqrt{n}\bar{s}\beta)^2\right],u\in\cQ_1\right\},\label{def:tcQ}
\end{align}
and $L_{\nabla P,1}$ and $\cQ_1$ are as defined in \cref{def:tLip-gradP-Q1} and \cref{r0-theta}, respectively.
\end{lemma}

\begin{proof}
Recall that for any $k\ge 0$, the subproblem in \cref{prox-AL-pb-compact} has an optimal solution $w^k_*$ (see \cref{w-star-k}), the initial iterate of \cref{alg:admm-cvx-1} for solving \cref{prox-AL-pb-compact} is $w^k$, and $P_{i,k}$, $0\le i\le n$, are strongly convex with modulus $1/[(n+1)\beta]$. By \cref{lem:uit-wt-bd} with $(P_i,\tilde{w}^*,\tilde{w}^0,\sigma)=(P_{i,k},w^k_*,w^k,1/[(n+1)\beta])$, we obtain that $\{w^{k,t+1}\}_{t\ge0}$ and $\{u_i^{k,t+1}\}_{t\ge0,1\le i\le n}$ stay in a set $\widetilde{\cQ}$ defined as
\begin{equation*}
\widetilde{\cQ} := \left\{v:\|v-w^k_*\|^2\le\frac{(n+1)^3\beta^2}{(1-q^2)}+ (n+1)\beta \sum_{i=1}^n\left(\rho_i\|w^k_* - w^k\|^2 + \frac{1}{\rho_i} \|\nabla P_{i,k}(w^k_*) - \nabla P_{i,k}(w^k)\|^2\right)\right\}.
\end{equation*}
Thus, to show the boundedness of all the inner iterates, it suffices to derive upper bounds for $\|w^k_* - w^k\|$ and $\|\nabla P_{i,k}(w^k_*) - \nabla P_{i,k}(w^k)\|$ that are independent of $k$. Since $w^k,w^k_*\in\cQ_1$, we have 
\begin{equation}\label{def:Lip-Pik-cQ1}
\|\nabla P_{i,k}(w^k_*) - \nabla P_{i,k}(w^k)\|\le L_{\nabla P,1}\|w^k_*-w^k\|
\end{equation}
due to \cref{def:tLip-gradP-Q}, where we note that $L_{\nabla P,1}$ is independent of $k$. Moreover, $\|w^k_* - w^k\|\le r_0+ 2\sqrt{n}\bar{s}\beta$ from \cref{upbd:wk-wstar} provides a $k$-independent upper bound for $\|w^k_* - w^k\|$. Thus, we have 
\begin{align}  \label{eq:bounded_Q2_key}
    \sum_{i=1}^n\left(\rho_i\|w^k_* - w^k\|^2 + \frac{1}{\rho_i} \|\nabla P_{i,k}(w^k_*) - \nabla P_{i,k}(w^k)\|^2\right) \le \sum_{i=1}^n\left[\left(\rho_i+\frac{L_{\nabla P,1}^2}{\rho_i}\right)(r_0+2\sqrt{n}\bar{s}\beta)^2\right]. 
\end{align}
Finally, combining the above results and noting that $w^k_*\in\cQ_1$ completes the proof.
\end{proof}

\paragraph{Lipschitz modulus of $\nabla P_{i,k}$ for all $0\le i\le n$ and all $k$ over $\cQ_2$}
Let $L_{\nabla f,2}$ be the Lipschitz constant of $\nabla f_i$, $1\le i\le n$, on $\cQ_2$, and $L_{\nabla c,2}$ be the Lipschitz constant of $\nabla c_i$, $0\le i\le n$, on $\cQ_2$. Also, define 
\begin{equation}
U_{\nabla c,2}:=\sup_{w\in\cQ_2}\max_{0\le i\le n} \|\nabla c_i(w)\|,\qquad U_{c,2}:=\sup_{w\in\cQ_2}\max_{0\le i\le n} \|c_i(w)\|.\label{def:tupbd-c-gc}
\end{equation}
Using similar arguments as for deriving $L_{\nabla P,1}$ in \cref{def:tLip-gradP-Q1}, we can see that $\nabla P_{i,k}$, $0\le i\le n$, are Lipschitz continuous on $\cQ_2$ with modulus $L_{\nabla P,2}$ defined as
\begin{equation}\label{def:tLip-gradP}
L_{\nabla P,2} := L_{\nabla f,2}+(\|\mu^*\| + r_0+2\sqrt{n}\bar{s}\beta + \beta U_{c,2}) L_{\nabla c,2}+\beta U_{\nabla c,2}^2+\frac{1}{(n+1)\beta}.
\end{equation}

\paragraph{Proof of \cref{thm:outer-complex}(b)}
Recall that \cref{thm:complexity-alg1}  has established the number of iterations of \cref{alg:admm-cvx-1} for solving each subproblem of \cref{alg:g-admm-cvx}. In the rest of this proof, we derive an upper bound for the total number of inner iterations for solving all subproblems of \cref{alg:g-admm-cvx}.

We see from \cref{lem:bd-iters-sbsolver} that all iterates generated by \cref{alg:admm-cvx-1} for solving \cref{prox-AL-pb-compact} lie in $\cQ_2$. Also, $\nabla P_{i,k}$, $1\le i\le n$, are $L_{\nabla P,2}$-Lipschitz continuous on $\cQ_2$. Therefore, by \cref{thm:complexity-alg1} with $(\tau,P_i,\sigma,L_{\nabla P},\tilde{w}^*,\tilde{w}^0)=(\tau_k,P_{i,k},1/[(n+1)\beta],L_{\nabla P,2},w^k_*,w^k)$ and the discussion at the end of \cref{apx:cplx-admm}, one has that the number of iterations of \cref{alg:admm-cvx-1} for solving \cref{prox-AL-pb-compact} during the $k$-th outer loop is no more than
\begin{equation}\label{def:Tk}
T_k := \left\lceil\frac{2\log(\tau_k/(n+1+ b_k))}{\log \tilde{q}_r}\right\rceil + 1
\end{equation}
where 
\begin{align*}
&\tilde{q}_r:= \max\left\{q,\frac{1}{1+\tilde{r}}\right\},\quad \tilde{r}:=\min_{1\le i\le n}\left\{\frac{\rho_i}{(n+1)\beta(\rho_i^2+2 L_{\nabla P,2}^2)}\right\},\\
&b_k:= \left(2\sqrt{(n+1)\beta}\sum_{i=1}^n(L_{\nabla P,2} +\rho_i)+\sqrt{2(\sum_{i=1}^n\rho_i)}\right)\\
&\qquad\qquad\cdot\Bigg\{\Bigg[\sum_{i=1}^n\left(\frac{\rho_i}{2}\|w_*^k- w^k\|^2 + \frac{1}{2\rho_i}\|\nabla P_{i,k}(w_*^k)-\nabla P_{i,k}(w^k)\|^2\right) \\
&\qquad\qquad\qquad\qquad+\frac{1}{1-q}\left(\frac{(n+1)^2\beta}{2}+\frac{1}{(n+1)\beta}\sum_{i=1}^n\frac{1}{\rho_i^2+2L_{\nabla P,2}^2}\right)\Bigg]^{1/2}+ (n+1)\sqrt{\frac{\beta}{2}}\Bigg\}.
\end{align*}
Plugging \cref{eq:bounded_Q2_key} into $\bar{b}$, we have that $b_k\le \bar{b}$, where
\begin{align*}
&\bar{b}:= \left(2\sqrt{(n+1)\beta}\sum_{i=1}^n(L_{\nabla P,2} +\rho_i)+\sqrt{2(\sum_{i=1}^n\rho_i)}\right)\\
&\cdot\left\{\left[\sum_{i=1}^n\frac{\rho_i^2+L_{\nabla P,1}^2}{2\rho_i}(r_0+2\sqrt{n}\bar{s}\beta)^2+\frac{1}{1-q}\left(\frac{(n+1)^2\beta}{2}+\frac{1}{(n+1)\beta}\sum_{i=1}^n\frac{1}{\rho_i^2+2L_{\nabla P,2}^2}\right)\right]^{1/2} + (n+1)\sqrt{\frac{\beta}{2}} \right\}, 
\end{align*}
which is independent of $k$. By $b_k\le\Bar{b}$, $\tau_k=\bar{s}/(k+1)^2$, $k\le K_{\eps_1,\eps_2}$ where $K_{\eps_1,\eps_2}$ is the upper bound for the number of outer iterations as defined in \cref{apx:cplx-pal}, and \cref{def:Tk}, one has that
\[
T_k \le \left\lceil\frac{2\log((n+1 + \bar{b})(K_{\eps_1,\eps_2}+1)^2/\bar{s})}{\log (\tilde{q}_r^{-1})}\right\rceil + 1.
\]
Therefore, by $K_{\epsilon_1,\epsilon_2}=\cO(\max\{\epsilon_1^{-2},\epsilon_2^{-2}\})$, one can see that the total number of inner iterations of \cref{alg:g-admm-cvx} is at most 
\begin{align}\label{def:specific-total-inner}
\sum_{k=0}^{K_{\eps_1,\eps_2}} T_k \le (K_{\eps_1,\eps_2}+1)\left(\left\lceil\frac{2\log((n+1 + \bar{b})(K_{\eps_1,\eps_2}+1)^2/\bar{s})}{\log (\tilde{q}_r^{-1})}\right\rceil+1\right)=\widetilde{\cO}(\max\{\epsilon_1^{-2},\epsilon_2^{-2}\}).
\end{align}
Hence, \cref{thm:outer-complex}(b) holds as desired.

\begin{algorithm}[!htbp]
\caption{A centralized proximal AL method for solving \cref{cvx-cnstr-central}}
\label{alg:c-prox-AL}
\noindent\textbf{Input}: tolerances $\epsilon_1,\epsilon_2\in(0,1)$, $w^0\in\dom(h)$, $\mu^0\ge0$, nondecreasing positive $\{\tau_k\}_{k \ge 0}$, and $\beta>0$.
\begin{algorithmic}
\For{$k=0,1,2,\ldots$}
\State Apply a centralized solver to find an approximate solution $w^{k+1}$ to:
\begin{equation*}
    \min_{w} \left\{\ell_k(w)= f(w) + h(w) + \frac{1}{2\beta}\left(\|[\mu^k+\beta c(w)]_+\|^2-\|\mu^k\|^2\right) + \frac{1}{2\beta}\|w-w^k\|^2\right\}    
\end{equation*}
\State such that 
\[
\rmdist_\infty(0,\partial\ell_k(w^{k+1}))\le \tau_k.
\]
\State Update the Lagrangian multiplier:
\[
\mu^{k+1} = [\mu^k + \beta c(w^{k+1})]_+.
\]
\State Output $(w^{k+1},\mu^{k+1})$ and terminate the algorithm if
\[
\|w^{k+1}-w^k\|_\infty+\beta\tau_k\le\beta\epsilon_1,\qquad\|\mu^{k+1}-\mu^k\|_\infty\le\beta\epsilon_2.
\]
\EndFor
\end{algorithmic}
\end{algorithm}

\section{A centralized proximal AL method}\label{apx:cpal}

In this part, we present a centralized proximal AL method (adapted from Algorithm~2 of \cite{LZ18AL}) for solving the convex constrained optimization problem:
\begin{equation}\label{cvx-cnstr-central}
\min_w\ f(w) + h(w)\quad\st\quad c(w)\le0,
\end{equation}
where the function $f:\bR^d\to\bR$ and the mapping $c:\bR^d\to\bR^m$ are continuous differentiable and convex, and $h$ is closed convex.

\section{Extra Numerical Results}\label{apx:ad-nr}
\subsection{Dataset description for Neyman-Pearson classification}\label{apx:dataset-desc}
In this part, we describe the datasets for Neyman-Pearson classification in \cref{sec:npc}. `breast-cancer-wisc', `adult-a', and `monks-1' are three binary classification datasets. We present the total number of samples for class 0 and class 1 and the number of features.
\begin{table}[h]
\caption{Datasets for Neyman-Pearson classification}
\smallskip
\centering
\begin{tabular}{c|c|c}
\hline
dataset& class 0/class 1 & feature dimension \\
\hline
breast-cancer-wisc &  455/240 & 20\\
adult-a & 24715/7840 & 21 \\
monks-1 & 275/275 & 21\\
\hline
\end{tabular}
\label{table:DD-NP}
\end{table}

\begin{table}
\centering
\caption{Numerical results for \cref{pb:lcqp} using our algorithm vs. using cProx-AL. Inside the parentheses are the respective standard deviations over 10 random trials.}
\smallskip
\resizebox{\textwidth}{!}{
\begin{tabular}{ccc||rrr||rr}
\hline
& & &\multicolumn{3}{c||}{objective value} & \multicolumn{2}{c}{feasibility violation} \\
$n$ & $d$ & $\tilde{m}$ & \multicolumn{1}{c}{\cref{alg:g-admm-cvx}}  & \multicolumn{1}{c}{cProx-AL} & \multicolumn{1}{c||}{relative difference} & \multicolumn{1}{c}{\cref{alg:g-admm-cvx}} & \multicolumn{1}{c}{cProx-AL} \\\hline
\multirow{3}{*}{1} & 100 & 1 & -0.23 (4.65e-6) & -0.23 (2.38e-5) & 1.63e-3 (1.01e-4) & 3.33e-4 (1.14e-5) & 5.68e-4 (2.82e-5) \\
& 300 & 3 & -0.37 (2.74e-6) & -0.37 (1.32e-6) & 1.01e-3 (3.51e-5) & 3.52e-4 (1.44e-5) & 4.45e-4 (1.70e-5) \\
& 500 & 5 & -0.30 (1.36e-5) & -0.30 (7.54e-6) & 1.34e-3 (4.62e-5) & 4.38e-4 (5.36e-5) & 3.85e-4 (1.05e-5) \\
\hline
\multirow{3}{*}{5} & 100 & 1 & 9.81 (7.18e-5) & 9.80 (1.46e-5) & 1.09e-3 (7.96e-6) & 1.34e-4 (9.02e-6) & 8.03e-4 (1.57e-6) \\ 
& 300 & 3 &  8.47 (8.12e-5) &  8.45 (1.30e-5) & 1.36e-3 (9.62e-6) & 1.09e-4 (1.31e-5) & 8.28e-4 (1.98e-6) \\
& 500 & 5 & 9.92 (4.43e-5) & 9.91 (4.87e-6) & 8.26e-4 (4.27e-6) & 1.33e-4 (9.68e-6) & 3.73e-4 (2.43e-7) \\
\hline
\multirow{3}{*}{10} & 100 & 1 & 49.40 (9.02e-5) & 49.37 (5.82e-6) & 5.59e-4 (1.67e-5) & 7.31e-5 (7.54e-6) & 5.88e-4 (1.34e-7) \\ 
& 300 & 3 & 41.49 (7.04e-5) & 41.44 (5.48e-6) & 1.14e-3 (1.77e-6) & 8.56e-5 (2.27e-7) & 8.73e-4 (7.26e-6) \\
& 500 & 5 & 41.45 (2.25e-5) & 41.41 (5.30e-6) & 9.39e-4 (4.94e-7) & 9.29e-4 (2.55e-6) & 7.66e-4 (1.37e-7) \\
\hline 
\end{tabular}
}
\label{table:FL-qp}

\end{table}
\subsection{Linear equality constrained quadratic programming}\label{sec:lcqp}
In this subsection, we consider the linear equality constrained quadratic programming:
\begin{equation}\label{pb:lcqp}
\min_{w}\ \sum_{i=1}^n\left(\frac{1}{2}w^TA_i w + b_i^Tw\right)\quad\st\quad C_iw+d_i=0,\quad 0\le i\le n,
\end{equation}
where $A_i\in\bR^{d\times d}$, $1\le i\le n$, are positive semidefinite, $b_i\in\bR^d$, $1\le i\le n$, $C_i\in\bR^{\tilde{m}\times d}$, $0\le i\le n$, and $d_i\in\bR^{\tilde{m}}$, $0\le i\le n$.

For each $(d, n, \tilde{m})$, we randomly generate an instance of \cref{pb:lcqp}. In particular, for each $1\le i\le n$, we first generate a random matrix $A_i$ by letting $A_i=U_iD_iU_i^T$, where $D_i\in\bR^{d\times d}$ is a diagonal matrix. The diagonal entries of $D_i$ are generated randomly from a uniform distribution over $[0.5, 1]$, and $U_i\in\bR^{d\times d}$ is a randomly generated orthogonal matrix. We then randomly generate $C_i$, $0\le i\le n$, with all entries drawn from a normal distribution with mean zero and a standard deviation of $1/\sqrt{d}$. Finally, we generate $b_i$, $1\le i\le n$ and $d_i$, $0\le i\le n$ as random vectors uniformly selected from the unit Euclidean sphere.

We apply \cref{alg:g-admm-cvx} and cProx-AL (\cref{alg:c-prox-AL}) to find a $(10^{-3},10^{-3})$-optimal solution of \cref{pb:lcqp}, and compare their solution quality. In particular, when implementing \cref{alg:g-admm-cvx}, we exactly solve the quadratic subproblems in \cref{sbpb:phi-gt,sbpb:phi-it}. We run 10 trials of \cref{alg:g-admm-cvx} and cProx-AL, where for each run, both algorithms share the same initial point $w^0$, randomly chosen from the unit Euclidean sphere. We set the other parameters for \cref{alg:g-admm-cvx} and cProx-AL as $\mu^0_i=(0,\ldots,0)^T\ \forall 0\le i\le n$, $\bar{s}=0.1$ and $\beta=10$. We also set $\rho_i=1$ for each $1\le i\le n$ for \cref{alg:admm-cvx-1}.

We observe that \cref{alg:g-admm-cvx} and cProx-AL are capable of finding nearly feasible solutions, and achieve similar objective value. In view of the small standard deviations, we observe that the convergence behavior of \cref{alg:g-admm-cvx} remains stable across 10 trial runs.
	
\end{document}

%% file: math_commands.tex
\usepackage{amsmath,amsfonts,bm}

\def\eqref#1{equation~\ref{#1}}

\def\1{\bm{1}}

\def\eps{{\epsilon}}

\DeclareMathAlphabet{\mathsfit}{\encodingdefault}{\sfdefault}{m}{sl}
\SetMathAlphabet{\mathsfit}{bold}{\encodingdefault}{\sfdefault}{bx}{n}

\def\cI{{\mathcal{I}}}
\def\cJ{{\mathcal{J}}}

\def\bR{{\mathbb{R}}}

\def\rmdist{{\mathrm{dist}}}
\def\dom{{\mathrm{dom}}}

\def\cL{{\mathcal{L}}}
\def\cQ{{\mathcal{Q}}}
\def\cB{{\mathcal{B}}}
\def\cC{{\mathcal{C}}}
\def\cT{{\mathcal{T}}}

\def\cN{{\mathcal{N}}}
\def\cO{{\mathcal{O}}}
\def\cU{{\mathcal{U}}}

\newcommand{\R}{\mathbb{R}}

\newcommand{\st}{\mathrm{s.t.}}

\DeclareMathOperator*{\argmin}{arg\,min}